\RequirePackage{iftex}
\ifpdftex
  \ifdefined\pdfvariable
    \pdfvariable minorversion=4  % pdfTeX 1.40+
  \else
  \fi
\fi
\documentclass[letterpaper,journal]{IEEEtran}
\usepackage{algorithmic}
\usepackage{algorithm}
\usepackage{array}
\usepackage[caption=false,font=footnotesize,labelfont=rm,textfont=rm]{subfig}
\usepackage{textcomp}
\usepackage{stfloats}
\usepackage{url}
\usepackage{verbatim}
\usepackage{graphicx}
\usepackage{pdfpages}
\usepackage{cite}
\usepackage{amsmath,amssymb,amsfonts}
\usepackage{amsthm}
\usepackage{xcolor,soul}
\usepackage{booktabs}
\usepackage{multirow}
\usepackage{makecell}
\usepackage{mathtools}

\usepackage{calligra}
\newtheorem{theorem}{Theorem}
\newtheorem{lemma}{Lemma}
\usepackage{comment}
\usepackage[colorlinks=true,linkcolor=blue, citecolor=green, urlcolor=magenta]{hyperref}

\allowdisplaybreaks
\DeclareMathOperator{\SO}{SO}
\DeclareMathOperator{\SE}{SE}
\newcommand{\soAlg}{\mathfrak{so}}
\DeclareMathOperator{\diag}{diag}
\DeclareMathOperator{\rank}{rank}

\DeclareMathOperator{\logm}{logm}
\DeclareMathOperator*{\argmax}{arg\,max}
\DeclareMathOperator*{\argmin}{arg\,min}

\newcommand{\bbR}{\mathbb{R}}

\newcommand{\reff}{r_{\mathrm{eff}}}
\newcommand{\Cov}{\mathbf{C}}
\newcommand{\bz}{\boldsymbol{z}}   % query (measurement)
\newcommand{\by}{\boldsymbol{y}}   % database / target (map)
\newcommand{\bp}{\boldsymbol{p}}
\newcommand{\bx}{\boldsymbol{x}}
\newcommand{\bV}{\boldsymbol{V}}
\newcommand{\bU}{\boldsymbol{U}}
\newcommand{\bSig}{\boldsymbol{\Sigma}}
\newcommand{\bM}{\boldsymbol{M}}
\newcommand{\bR}{\boldsymbol{R}}
\newcommand{\bOm}{\boldsymbol{\Omega}}
\newcommand{\bZ}{\boldsymbol{Z}}   % query matrix
\newcommand{\bY}{\boldsymbol{Y}}   % database / target matrix

\begin{document}

\title{\textsc{Rosetta}: Rotation Subspace Estimation\\for Training-free Transmodal Alignment}

\author{Alex~Junho~Lee%
\thanks{Submitted to IEEE Transactions on Image Processing, March 1, 2026.}
\thanks{Alex~Lee is with the School of Mechanical Systems Engineering,
Sookmyung Women's University, Seoul 04310, South Korea
(alexlee@sm.ac.kr).}%
}

\markboth{Submitted to IEEE Transactions on Image Processing, March 1, 2026}
{Lee: \textsc{Rosetta} --- Rotation Subspace Estimation for Training-free Transmodal Alignment}

\maketitle

% =====================================================================
%  ABSTRACT
% =====================================================================
\begin{abstract}
When a query must be matched against a database built with a
different sensor or processed by a different encoder, the independently
trained encoders impose distinct
coordinate systems on the feature space.
No canonical reference frame
exists to which all encoders can be aligned, and no single training procedure has fully closed this gap.
We show that the relative transformation between any two such systems
is a spatially preserving rotation in whitened
feature space, whose angular velocity
exhibits white-noise characteristics: large in norm, temporally
uncorrelated, and unpredictable from the features themselves.
However, the cross-covariance between encoder outputs has low effective
rank, constraining only a small fraction of the rotation parameters,
therefore local alignment is tractable.
\textsc{Rosetta} is a training-free calibration layer that estimates
the local rotation via an expectation-maximization procedure,
requiring no co-observation at query time.
We validate \textsc{Rosetta} on three public driving datasets
spanning LiDAR, camera, thermal, and event
modalities in cross-modal, cross-device,
cross-representation, and cross-backbone settings, with milliseconds overhead per query.
\end{abstract}


% =====================================================================
%  I.  INTRODUCTION
% =====================================================================
\section{Introduction}
\label{sec:introduction}

\IEEEPARstart{V}{isual} place recognition achieves high
within-modality accuracy with modern encoders, including vision
foundation models such as DINOv3~\cite{DINOv3} for cameras
and sparse 3D convolution networks~\cite{MinkLoc3D} for LiDAR.
Cross-modal retrieval, where a LiDAR query is matched against a
camera-built database or vice versa, arises whenever the sensor
available at query time differs from the one used to build the
reference map, for example in heterogeneous robot fleets sharing a
single spatial database, cross-platform loop closure between vehicles
with different sensor suites, or legacy maps that must remain usable
after a sensor upgrade.
However, cross-modal retrieval accuracy drops to a
dataset-specific ceiling (Supplementary Sec.~B).
This degradation is not attributable to insufficient training:
jointly trained contrastive models retain a persistent modality
gap even after convergence~\cite{Liang2022MindGap},
and independently trained embedding spaces are in general
non-isomorphic~\cite{Vulic2020Isomorphic,Ormazabal2019Limitations},
so that no single global transformation can bridge
them~\cite{Glavas2020NonLinear}.

% =====================================================================
%  MAIN FIGURE (first page, top)
% =====================================================================
\begin{figure}[!t]
	\centering
	\includegraphics[width=0.85\linewidth]{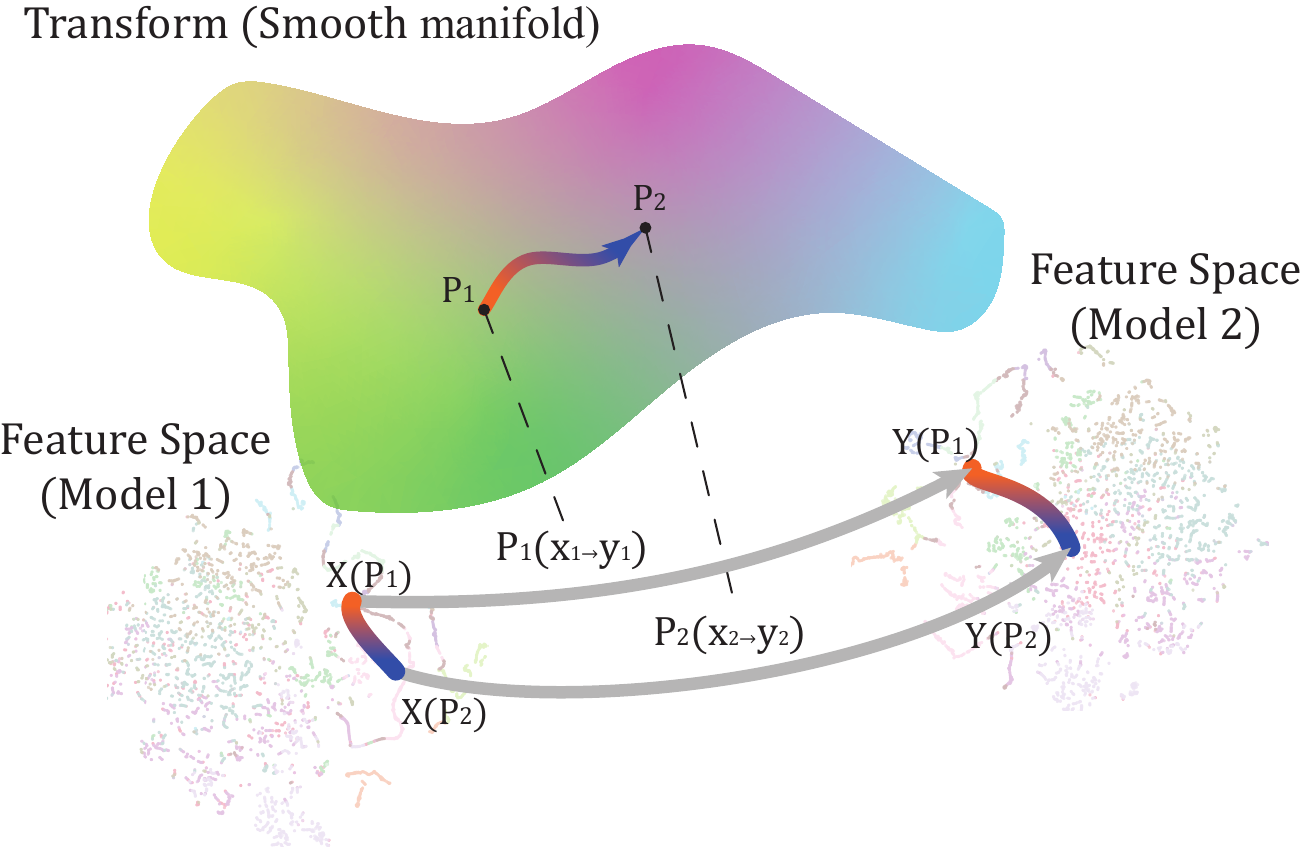}
	\caption{%
		Two independently trained models map locations
		$P_1, P_2$ on a smooth manifold (\emph{top}) to different
		coordinates $X(P_i)$ and $Y(P_i)$ in separate feature
		spaces (\emph{bottom}).
		The gray arrows indicate that the inter-model
		transformation $P_i\!(x_i{\to}y_i)$ differs between
		locations: the rotation at~$P_1$ does not equal
		the rotation at~$P_2$.
		Because $\bR(\bx)\in\SO(K)$ varies with position,
		no single global rotation can align the two spaces.
	}
	\label{fig:main}
\end{figure}

We show that this ceiling has a geometric origin.
Each encoder imposes its own coordinate system on the feature space;
after whitening, the inter-encoder misalignment at a given location
reduces to a rotation $\bR(\bx)\in\SO(K)$.
This rotation is non-stationary: because each encoder responds
independently to scene content, and scene content varies
quasi-randomly along a trajectory, the relative rotation inherits
this randomness, exhibiting white-noise angular velocity in the Lie
algebra $\soAlg(K)$.
The alignment at one location carries no predictive information about
the alignment at another, so no single global rotation can align
the two spaces (Sec.~\ref{sec:nonstationarity}).

However, the same geometric structure that renders global alignment
ill-posed makes local alignment tractable.
The cross-covariance between encoder outputs has low effective rank
$\reff \ll K$: most of the $K$ feature dimensions encode
modality-specific details (texture for cameras, geometry for LiDAR),
while only a low-dimensional subspace carries the place-level
information shared across modalities.
Because $\reff$ grows sublinearly with spatial coverage,
the number of signal-carrying rotation parameters remains small
within any local window, so a handful of correspondences suffices
to determine the rotation.
Finding these correspondences poses a circular dependency:
accurate cross-encoder matching requires the rotation, yet estimating
the rotation requires matched pairs.
Because the alignment depends only on the geometric relationship
between encoder outputs, it applies identically whether the encoders
differ in modality, sensor hardware, backbone architecture, or
training procedure; we therefore use the term \emph{transmodal},
and these settings are evaluated in Sec.~\ref{sec:experiments}.

We propose \textbf{\textsc{Rosetta}},
a training-free calibration layer that estimates the local Procrustes
rotation from coarse spatial priors alone, without modifying either
encoder.
A local Soft~EM procedure iteratively refines soft
correspondences and the rotation estimate; at query time,
only a coarse position prior is needed
(Sec.~\ref{sec:method}).
The aligned descriptors yield cross-modal correspondences that
serve as relative factors in a pose graph, connecting feature-space
alignment to position estimation
(Sec.~\ref{sec:graph_registration}).

\begin{figure*}[!t]
	\centering
	\includegraphics[width=\textwidth]{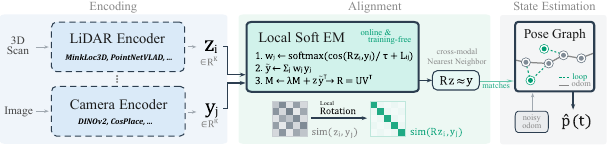}
	\caption{%
		Overview of the \textsc{Rosetta} pipeline.
		(Left)~Frozen encoders map each observation to
		$\bz_i, \by_j \in \bbR^K$; LiDAR and camera
		encoders are shown, but the method applies to any
		modality or backbone.
		(Center)~A local Soft~EM
		(Algorithm~\ref{alg:sliding}) computes soft weights
		$w_j$ from a spatial prior $L_{ij}$, accumulates the
		cross-covariance $\bM$, and recovers
		$\bR\!\in\!\SO(K)$ via SVD.
		Bottom: the cross-encoder similarity sharpens
		to diagonal after rotation.
		(Right)~The rotated descriptor
		$\bR\bz_i \!\approx\! \by_j$ yields cross-modal
		nearest-neighbor matches that enter a pose graph
		as loop-closure factors, producing
		the trajectory estimate $\hat{p}(t)$.}
	\label{fig:overview}
\end{figure*}

The principal contributions are as follows:
\begin{enumerate}
  \item \textbf{Observability and scaling law.}
    The cross-covariance between any two independently trained
    encoders has low effective rank $\reff \ll K$ that grows
    sublinearly with spatial coverage, so the local rotation
    is overdetermined from few observations
    (Secs.~\ref{sec:nonstationarity},~\ref{sec:theory}).

  \item \textbf{Training-free alignment layer.}
    \textsc{Rosetta} estimates the local $\SO(K)$ rotation
    between any two frozen encoders from coarse spatial
    priors alone, without retraining or modifying
    either encoder (Sec.~\ref{sec:method}).

  \item \textbf{Cross-representation registration.}
    The aligned descriptors provide cross-encoder correspondences
    that enter a pose graph as relative factors, enabling
    localization in cross-modal, cross-device,
    cross-representation, and cross-backbone settings
    (Sec.~\ref{sec:graph_registration},~\ref{sec:experiments}).
\end{enumerate}

Sec.~\ref{sec:related} reviews related work.
Sec.~\ref{sec:formulation} formalizes the gauge model
and non-stationarity.
Secs.~\ref{sec:method}--\ref{sec:conclusion} present the
alignment algorithm, experiments on KITTI-360, VIVID, and HeLiPR,
and discussion.

% =====================================================================
%  II.  RELATED WORK
% =====================================================================
\section{Related Works}
\label{sec:related}

\subsection{Place Recognition}
Place recognition has reached high accuracy within single modality:
frozen foundation models such as DINOv3~\cite{DINOv3} exceed $99\%$
recall on standard benchmarks~\cite{LowrySurveyVPR,NetVLAD,CosPlace,EigenPlaces}.
Cross-modal retrieval, by contrast, has not matched this performance.
Existing methods bridge the modality gap either by converting both
inputs to a shared representation, such as 2.5D range-depth
images~\cite{LC2}, bird's-eye-view projections~\cite{I2PMatching},
or DINOv2-processed range images~\cite{ImLPR2025}, or by learning a
joint embedding through contrastive objectives~\cite{VXP} or
vision-to-LiDAR distillation~\cite{ScaLR}.
In both cases the alignment is tied to the specific training
configuration and requires at least partial fine-tuning to transfer
to new domains.
In practice, however, retraining a foundation model is
impractical, and even fine-tuning assumes access to the target domain;
in many deployments the environment or sensor suite is completely
unseen at training time.
Domain adaptation methods such as CORAL~\cite{SunCORAL2016} and
multi-sensor SLAM systems~\cite{CadenaSLAMSurvey} similarly require
modifying the encoder to produce aligned features.
When both encoders are frozen and were trained independently, the
inter-encoder alignment varies with position, so no single global
transformation suffices (Sec.~\ref{sec:nonstationarity});
\textsc{Rosetta} calibrates this alignment post hoc by estimating
a local rotation from coarse spatial priors, without retraining
or modifying either encoder.

\subsection{Representation Alignment}
\label{sec:related_alignment}
The orthogonal Procrustes problem~\cite{Schoenemann1966} finds the
nearest rotation between two point configurations.
This global solution succeeds when the two embedding spaces share
sufficient structural similarity, a condition typically met within
a single modality or task. For example, cross-lingual word embeddings trained on comparable
corpora are approximately isomorphic~\cite{ConneauWord2018}, and
encoders trained with the same loss on the same data produce
quasi-isometric latent spaces whose similarity increases with
model capacity~\cite{Moschella2023Relative,Huh2024Platonic,Maystre2025Procrustes}.
However, non-linear misalignments that Procrustes cannot capture
persist~\cite{Vejendla2025DriftAdapter} even within these favorable settings.
Across modalities, Liang~et~al.~\cite{Liang2022MindGap} show that
even jointly trained contrastive encoders retain a persistent
modality gap caused by initialization geometry and the contrastive
objective.
For independently trained encoders the gap widens further,
and no single global linear map suffices to bridge
it~\cite{Ormazabal2019Limitations,Glavas2020NonLinear}.
Vuli\'{c}~et~al.~\cite{Vulic2020Isomorphic} show that the degree
of non-isomorphism varies across encoder pairs and correlates with
training conditions, suggesting the need for a per-pair diagnostic.
The local $\SO(K)$ rotation that we formalize in
Sec.~\ref{sec:gauge} offers a principled per-pair diagnostic.

All prior methods assume a single spatially invariant
alignment~\cite{Fumero2024LatentFM}, and even methods that account for spatial
variation~\cite{Bazeille2019fmralign,Andreella2023ProMises}
rely on a known domain geometry.
For encoders processing observations along a spatial trajectory,
neither assumption holds: the scene content changes continuously
along the trajectory, and each encoder responds to it differently,
making the alignment position-dependent
(Sec.~\ref{sec:gauge}), and $\bR(\bx)$ drifts continuously along the trajectory
(Sec.~\ref{sec:nonstationarity}).
\textsc{Rosetta} addresses this setting by recovering the local
rotation from coarse spatial priors alone, without requiring
co-observations at query time (Sec.~\ref{sec:method}).

% =====================================================================
%  III.  PROBLEM ANALYSIS AND THEORY
% =====================================================================
\section{Low-Rank Non-Stationary Alignment}
\label{sec:formulation}

This section formalizes the cross-encoder retrieval problem and
establishes two results that motivate the algorithm of
Sec.~\ref{sec:method}
(notation is summarized in Supplementary Table~A4).
First, the inter-encoder alignment $\bR(\bx)\!\in\!\SO(K)$ is
non-stationary, with angular velocity that behaves as white noise
in the Lie algebra $\soAlg(K)$, so no single global rotation can
align the two spaces (Sec.~\ref{sec:nonstationarity}).
Second, the cross-covariance between encoder outputs has low
effective rank $\reff \ll K$, which limits the number of rotation
parameters that must be estimated and makes local alignment
overdetermined from few observations
(Sec.~\ref{sec:theory}).

\subsection{Rotation Model}
\label{sec:setup}
Let $f_{\mathcal{A}}: \mathcal{X}_{\mathcal{A}} \to \bbR^D$ and
$f_{\mathcal{B}}: \mathcal{X}_{\mathcal{B}} \to \bbR^D$ be two
independently trained feature extractors.
A \emph{database} $\{(\by_j, \boldsymbol{q}_j)\}_{j=1}^{M}$ stores
Principal Component Analysis (PCA)-projected
target features $\by_j \in \bbR^K$
at known positions~$\boldsymbol{q}_j$.
A \emph{query} produces $\bz \in \bbR^K$
at an unknown position.  The retrieval task is to find the database index
\begin{equation}\label{eq:retrieval}
  j^* = \argmax_{j} \;\cos\!~\bigl(\bR\,\bz,\; \by_j\bigr),
\end{equation}
where $\bR \in \SO(K)$ is the alignment rotation to be determined.

Each encoder output is PCA-projected to $\bbR^K$ ($K{=}128$)
and L2-normalized to the unit sphere $\mathcal{S}^{K-1}$.
PCA centers each modality and reduces the ambient dimension,
making the $K{\times}K$ singular value decomposition (SVD) in Algorithm~\ref{alg:sliding}
tractable.
L2 normalization removes residual magnitude differences,
so the cosine similarity in Eq. \eqref{eq:retrieval} depends
only on the rotational alignment $\bR$.

As place-recognition encoders, $f_{\mathcal{A}}$ and
$f_{\mathcal{B}}$ map nearby locations to similar descriptors
and distant locations to dissimilar ones by design.
Since both encoders preserve pairwise similarity,
the inter-encoder mapping also preserves pairwise similarity.
Both encoders also observe the same physical locations, so their
feature distributions share dominant variance directions and
independent PCA projections to $\bbR^K$ preserve this relationship.
To isolate the rotational component of this mapping,
the centering step of PCA removes the mean offset between modalities,
and the projection onto the principal components decorrelates each feature space.
L2 normalization then maps each descriptor to $\mathcal{S}^{K-1}$,
removing magnitude differences.
With translation, covariance, and scale removed, the remaining
transformation is similarity-preserving on $\mathcal{S}^{K-1}$,
and the isometries of $\mathcal{S}^{K-1}$ are precisely the
orthogonal transformations.
The Procrustes solution (Eq.~\eqref{eq:procrustes}) restricts this
to $\SO(K)$, yielding the \emph{rotation model}.
PCA centering already absorbs the translational
component of $\SE(K)$, i.e. the modality
gap~\cite{Liang2022MindGap}; subtracting a per-window
mean on top of the global PCA mean changes recall by
less than $0.01$, so the residual translation is
negligible.
Empirically, independently trained encoders produce approximately
isometric feature spaces~\cite{Moschella2023Relative,Huh2024Platonic},
supporting this model.
Limitations, including the unique-descriptor assumption and the
empirical Procrustes residual ($0.36$ to $0.43$ across datasets),
are discussed in Sec.~\ref{sec:discussion}.

A single global rotation would suffice if all scenes induced the
same inter-encoder map.
The next two subsections show that this is not the case.
Different scenes induce different encoder-specific responses,
making $\bR$ position-dependent.

\subsection{Low-Rank Cross-Covariance}
\label{sec:gauge}
At spatial position $\bx$, the \emph{local alignment}
$\bR(\bx) \in \SO(K)$ is given by the orthogonal Procrustes solution
\begin{equation}\label{eq:procrustes}
  \bR(\bx) = \bU\,\diag(1,\ldots,1,\det(\bU\bV^\top))\,\bV^\top,
\end{equation}
where $\bU\bSig\bV^\top = \bM(\bx)$ is the SVD of the local
cross-covariance
\begin{equation}\label{eq:crosscov}
  \bM(\bx) = \bZ(\bx)^\top \bY(\bx) \;\in\; \bbR^{K \times K},
\end{equation}
where $\bZ(\bx), \bY(\bx) \in \bbR^{W \times K}$ stack the query and
database features within a spatial window of $W$ frames around~$\bx$.

In the PCA-projected feature space, $\bR(\bx)$ maps the query
coordinate frame to the database coordinate frame at position~$\bx$.
We use the term \emph{gauge} to emphasize that each encoder's
coordinate system is arbitrary, and only the relative alignment between the two can be estimated from data.
The Procrustes solution recovers $\bR(\bx)$ from the cross-covariance
$\bM(\bx)$ by the same SVD procedure that estimates $\SO(3)$
rotations in 3D point-cloud registration, extended here to $\SO(K)$.
The key difference is that $\bM$ is severely
rank-deficient ($\reff \ll K$), so most rotation parameters are
determined by orthogonality rather than data
(Sec.~\ref{sec:dof}).

\subsubsection{Cross-Covariance Decomposition}
\label{sec:decomposition}
When the two feature spaces are approximately
isometric, a property empirically observed in foundation
models~\cite{Moschella2023Relative,Huh2024Platonic}, the local
relationship between query and database features within a spatial
window can be modeled as
\begin{equation}\label{eq:local_model}
  \bz_i \;=\; \bR\,\by_i + \boldsymbol{n}_i,
\end{equation}
where $\bR \in \SO(K)$ is assumed constant within the window and
$\boldsymbol{n}_i$ is a zero-mean residual uncorrelated with $\by_i$.
Substituting Eq.~\eqref{eq:local_model} into
Eq.~\eqref{eq:crosscov} decomposes the cross-covariance into
signal and noise.
\begin{lemma}[Signal--Noise Decomposition]
\label{lem:factorization}
The cross-covariance decomposes as
\begin{equation}\label{eq:decomp}
  \bM \;=\; \bR\,\Cov_y \;+\; \boldsymbol{E},
\end{equation}
where $\Cov_y = \bY^\top\bY$ is the target self-covariance and
$\boldsymbol{E} = \boldsymbol{N}^\top\bY$ is the residual.
When the residual is small in Frobenius norm,
$\|\boldsymbol{E}\|_F \ll \|\bR\,\Cov_y\|_F$, the effective rank satisfies
$\reff(\bM) \approx \reff(\Cov_y)$.
\end{lemma}
\noindent\textit{Proof sketch.}\quad
Since $\bR$ is orthogonal, $\bR\,\Cov_y$ and $\Cov_y$ share the
same singular values, so
$\reff(\bR\,\Cov_y) = \reff(\Cov_y)$.
When $\|\boldsymbol{E}\|_F$ is negligible, the singular values of
$\bM$ are close to those of $\bR\,\Cov_y$, giving
$\reff(\bM) \approx \reff(\Cov_y)$.
Empirically, $\reff(\bM)/\reff(\Cov_y) \in [0.98,\, 1.02]$
across all three datasets (Sec.~\ref{sec:experiments}).
See Supplementary Sec.~A-1 for the full derivation.

\subsubsection{Degrees of Freedom}
\label{sec:dof}
A rotation in $K$ dimensions has $K(K{-}1)/2$ free parameters,
one for each pair of coordinate axes.
Because $\bM$ has effective rank $r \ll K$, the cross-covariance
only constrains how $\bR$ acts on the $r$-dimensional signal
subspace spanned by the top-$r$ singular vectors.
The action of $\bR$ on the remaining $K{-}r$ dimensions
is not constrained by the cross-covariance.
The signal subspace therefore determines $r(2K{-}r{-}1)/2$
of the total $K(K{-}1)/2$ rotation parameters
(Theorem~\ref{thm:identifiability}).
The orthogonality constraint $\bR^\top\bR = \boldsymbol{I}$
then propagates the known signal-subspace action to the null
space, fixing the remaining parameters.

The condition $\reff \ll K$ has a physical origin.
Each encoder's $K$-dimensional output captures all aspects of a scene:
texture, color, and lighting for cameras, and precise geometry and range
distributions for LiDARs.
The cross-covariance $\bM$ responds only to the subspace where
both modalities carry correlated information, namely the
place-level spatial layout (scene category, object configuration).
This shared subspace is low-dimensional, hence $\reff \ll K$.

When $\reff \ll K$, most parameters of $\bR \in \SO(K)$ lie in the
null space of $\bM$ and are resolved only by the orthogonality
constraint $\bR^\top\bR = \boldsymbol{I}$.
Theorem~\ref{thm:perturbation} (Sec.~\ref{sec:thm_perturb})
shows that these null-space components do not affect retrieval.
Features projected onto vanishing singular vectors carry near-zero
energy, so any orthogonal completion yields the same cosine
similarity.

\subsubsection{Empirical Rank of the Cross-Covariance}
\label{sec:rotation_spectrum}
The rank structure of $\bM$ can be quantified by the
effective rank~\cite{Roy2007EffectiveRank}, which is the exponential
Shannon entropy of the normalized singular value spectrum
$\{\sigma_k\}_{k=1}^{K}$, as follows:
\begin{equation}\label{eq:reff}
  \reff(\bM) = \exp~\!\Bigl(-\sum_{k=1}^{K} \bar{\sigma}_k
  \ln \bar{\sigma}_k\Bigr), \quad
  \bar{\sigma}_k = \frac{\sigma_k}{\sum_j \sigma_j}\,,
\end{equation}
With a window of $W{=}20$ frames
(dependence on spatial coverage in Supplementary Sec.~A-6)
and the encoder pair MinkLoc3Dv2~\cite{MinkLoc3Dv2}$\,\to\,$DINOv2-ViT-L~\cite{DINOv2},
chosen as a well-trained foundation model pair,
the spectrum is sharply concentrated across all three datasets:
$\sigma_1$ accounts for most of the Frobenius norm,
and $\reff$ ranges from $1.5$ (KITTI) through $5.1$ (VIVID)
to $9.9$ (HeLiPR), confirming $\reff \ll K{=}128$
(Fig.~\ref{fig:svd_spectrum}).

Since $\reff \ll K$, the local Procrustes solution determines
$\bR(\bx)$ from a small signal subspace at each position.
It remains to determine whether a single global $\bR$ suffices
or whether it must be estimated per window.

\begin{figure}[!t]
	\centering
	\subfloat[]{%
		\includegraphics[width=0.48\linewidth]{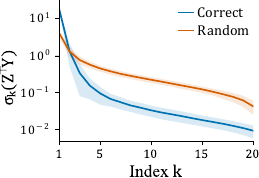}\label{fig:svd_traj}}
	\subfloat[]{%
		\includegraphics[width=0.48\linewidth]{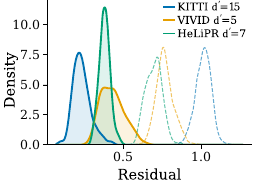}\label{fig:svd_kde}}
	\caption{Cross-covariance analysis ($W\!=\!20$, $K\!=\!128$).
		\protect\subref{fig:svd_traj}~Singular-value spectrum of
		$\mathbf{Z}^\top\!\mathbf{Y}$ on KITTI-360: correct pairings
		(blue) concentrate energy in $\sigma_1$; random pairings (orange)
		produce a flat spectrum.
		\protect\subref{fig:svd_kde}~Distribution of the Procrustes
		residual $\|\mathbf{Z}{-}\mathbf{Y}\mathbf{R}^\top\|_F$
		for correct (solid) and random (dashed) pairings on all three
		datasets. $d'$ denotes the separation between the two
		distributions.}
	\label{fig:svd_spectrum}
\end{figure}

\subsection{Non-Stationary Alignment}
\label{sec:nonstationarity}
A single global $\bR$ would suffice if the cross-covariance
structure were constant across positions.
In practice, as a window slides along the traversal (indexed
by~$t$), the top singular vectors
$(\boldsymbol{u}_1(t), \boldsymbol{v}_1(t))$ of $\bM(t)$
rotate substantially from one window to the next,
indicating that the alignment direction, not merely its
magnitude, varies with position.
Fig.~\ref{fig:nonstationarity} shows the retrieval matrix on VIVID
(MinkLoc3Dv2\,$\to$\,DINOv2-ViT-L).
Without rotation alignment, cross-modal descriptors after PCA and
L2 normalization yield only $0.04$ diagonal retrievals.
Per-window alignment raises this to $0.99$, while a single global
rotation achieves only ${\sim}0.47$ (Sec.~\ref{sec:global_fails}),
illustrating that $\bR(\bx)$ must be estimated locally.

\subsubsection{Local Rotation Dynamics}
\label{sec:lie_dynamics}
\label{thm:nonstationary}
The angular velocity between consecutive windows is defined as:
\begin{equation}\label{eq:omega}
  \bOm(t) = \logm~\!\bigl(\bR(t)^\top \bR(t{+}1)\bigr) \;\in\; \soAlg(K).
\end{equation}
Empirically, $\bOm(t)$ behaves as high-dimensional white noise.
The norm $\|\bOm(t)\|$ averages $13.1$\,rad/frame across all three
datasets, and the direction autocorrelation is
$-0.018 \pm 0.03$, which is consistent with zero.
Vectorizing $\bOm(t)$ into its $K(K{-}1)/2$ independent
upper-triangular entries gives
$\reff(\bOm) \approx K(K{-}1)/3$, meaning the rotation change
is not confined to a low-dimensional subspace but spreads across
roughly two-thirds of the available dimensions.
Canonical correlation analysis (CCA)~\cite{Hotelling1936CCA}
between $\bOm(t)$ and the feature difference $\Delta\bz(t)$
also measures only $0.035$, indicating negligible linear dependence
between the rotation change and the observed features.

While $\bR(t)$ itself is computable from the current window via
Eq.~\eqref{eq:procrustes}, the change $\bOm(t)$ cannot
be anticipated from features alone, necessitating per-window
estimation.
Additional evidence, including autocorrelation
(Supplementary Fig.~A2) and pairwise cosine heatmaps
(Supplementary Fig.~A3), is provided in
Supplementary Secs.~A-2, C, and~J.

\begin{figure}[!t]
	\centering
	\vspace{-4mm}
	\subfloat[]{%
		\includegraphics[height=120pt]{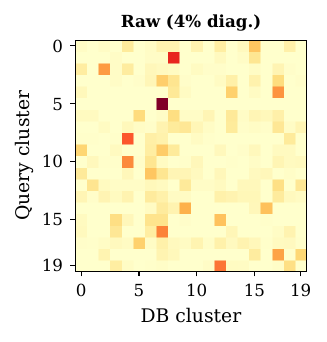}\label{fig:nonstat_raw}}%
	\hfill
	\subfloat[]{%
		\includegraphics[height=120pt]{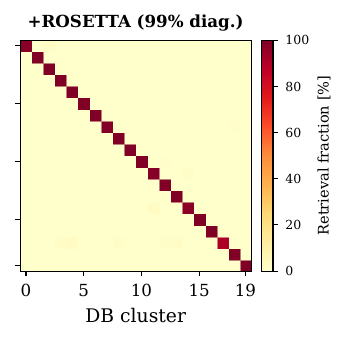}\label{fig:nonstat_aligned}}
	\caption{Cross-modal retrieval matrix on VIVID
		(MinkLoc3Dv2\,$\to$\,DINOv2-ViT-L).
		Each cell shows the fraction of queries from one spatial cluster
		retrieved from a given database cluster; a perfect retriever
		concentrates all mass on the diagonal.
		\protect\subref{fig:nonstat_raw}~Raw descriptors yield 0.04
		diagonal, comparable to random retrieval.
		\protect\subref{fig:nonstat_aligned}~After \textsc{Rosetta} alignment, 0.99 is on the diagonal.}
	\label{fig:nonstationarity}
\end{figure}

\subsubsection{Global Alignment Bound}
\label{sec:global_fails}
\label{cor:global_bound}
The global Procrustes solution fits a single rotation
$\bR^* = \argmin_{\bR} \sum_t \|\bR\,\bz(t) - \by(t)\|^2$
to all windows simultaneously, averaging over local rotations
that differ substantially from one another.
Since $\|\bOm\| \gg 0$ with zero autocorrelation,
the local rotations disperse widely through $\SO(K)$
without systematic bias, and their average loses directional
information, just as averaging uniformly distributed unit
vectors on a sphere yields the zero vector.
Accordingly, the global solution achieves only ${\sim}0.47$
on KITTI-360, far below the $0.99$ obtained by per-window
alignment
(Sec.~\ref{sec:exp_main}; Supplementary Secs.~A-3 and~G).
A training sweep over twenty configurations confirms that this
ceiling is intrinsic to the non-stationarity, not an artifact of
any particular training recipe
(Sec.~\ref{sec:exp_negative}, Table~\ref{tab:negative}).

Having established that global alignment fails due to
non-stationarity, the next section shows that local
alignment is, in contrast, overdetermined.

\subsection{Observability of Local Alignment}
\label{sec:theory}

\subsubsection{Signal and Null Subspace Decomposition}
\label{sec:thm_ident}

\begin{theorem}[Orthogonal Completion]
\label{thm:identifiability}
Let $\bM \in \bbR^{K \times K}$ have effective rank
$r = \lceil \reff(\bM) \rceil$ (ceiling of the effective rank,
Eq.~\eqref{eq:reff}).
The Procrustes solution $\bR \in \SO(K)$ has $K(K{-}1)/2$ free
parameters, of which exactly $r(2K{-}r{-}1)/2$ are determined by
the top-$r$ singular triplets of~$\bM$ (\emph{signal subspace}),
and the remaining $(K{-}r)(K{-}r{-}1)/2$ are fixed by
orthogonality alone (\emph{null-space completion}).
\end{theorem}
The proof partitions the $K(K{-}1)/2$ rotation parameters by the
subspaces they couple: signal-to-signal, signal-to-null, and null-to-null.
The first two groups are constrained by the top-$r$ singular
triplets, yielding $r(2K{-}r{-}1)/2$ data-determined parameters;
the third group is fixed by $\bR^\top\bR = \boldsymbol{I}$.
The full derivation is given in Supplementary Sec.~A-4.
For $K{=}128$ (total DOF $= 8128$), the data-determined parameters
amount to only $253$ ($3.1\%$) on KITTI ($\reff{=}1.5$),
$747$ ($9.2\%$) on VIVID ($\reff{=}5.1$),
and $1225$ ($15.1\%$) on HeLiPR ($\reff{=}9.9$).
The vast majority of the rotation is thus fixed by orthogonality
alone.

\subsubsection{Anisotropic Sensitivity}
\label{sec:thm_perturb}

\begin{theorem}[Anisotropic Sensitivity]
\label{thm:perturbation}
Under additive perturbation $\boldsymbol{E}$ on $\bM$, the
signal-subspace component of $\bR$ is stable, while the null-subspace
component is poorly determined but irrelevant to retrieval.
Specifically, the induced perturbation $\delta\bR$ satisfies
\begin{align}
  \|\delta\bR_{\mathrm{signal}}\| &\;\lesssim\;
    \frac{\|\boldsymbol{E}_r\|}{\sigma_r - \sigma_{r+1}},
    \label{eq:perturb_signal} \\
  \|\delta\bR_{\mathrm{null}}\| &\;\sim\;
    \frac{\|\boldsymbol{E}\|}{\epsilon},
    \label{eq:perturb_null}
\end{align}
where $\sigma_r$ is the $r$-th singular value and $\epsilon \to 0$
for the null singular values.
The spectral gap $\sigma_r - \sigma_{r+1}$ stabilizes the signal
component, per standard SVD perturbation
theory~\cite{HornJohnsonMatrix}.
Although $\|\delta\bR_{\mathrm{null}}\|$ diverges as
$\epsilon \to 0$, the features carry near-zero energy along
the null directions ($\|\bz_{\perp}\| \ll \|\bz_{\parallel}\|$),
so the retrieval error is bounded by the spectral ratio:
\begin{equation}\label{eq:null_irrelevant}
  \bigl|\cos(\bR\,\bz,\,\by) - \cos(\tilde{\bR}\,\bz,\,\by)\bigr|
  \;\leq\; \mathcal{O}\!\Bigl(\frac{\sigma_{r+1}}{\sigma_1}\Bigr).
\end{equation}
\end{theorem}
The proof decomposes $\bz$ into signal and null components and
shows that the retrieval error is controlled by the spectral ratio
$\sigma_{r+1}/\sigma_1$.
The full derivation is given in Supplementary Sec.~A-5.

\subsubsection{Sublinear Rank Growth}
\label{sec:thm_scaling}

\begin{figure}[!t]
	\centering
	\includegraphics[width=0.8\linewidth]{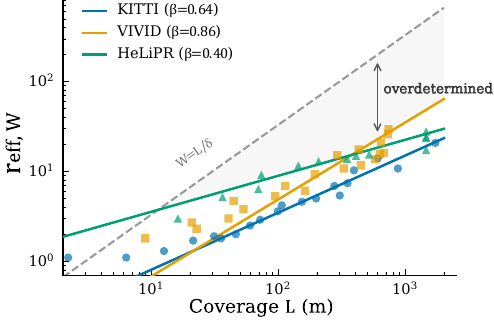}
	\caption{Effective rank $\reff$ versus window coverage~$L$.
		Colored lines: power-law fit
		$\reff = a\,L^{\beta}$.
		Gray dashed: number of paired observations
		$W = L/\delta$, where $\delta$ is the spatial
		sampling interval.
		Since $\beta < 1$, $W$ grows faster than $\reff$,
		so local alignment is overdetermined for all tested~$L$
		(Sec.~\ref{sec:cor_local}).}
	\label{fig:scaling}
	\vspace{-3mm}
\end{figure}

Across three datasets (KITTI-360, VIVID, HeLiPR) and two modality pairs,
the effective rank of the cross-covariance grows sublinearly with spatial
coverage~$L$:
\begin{equation}\label{eq:scaling}
  \reff \;=\; a\,L^{\beta}, \qquad \beta < 1.
\end{equation}
The exponent $\beta$ and coefficient $a$ are properties of the deployment
(sensor, encoder, environment).  The constraint $\beta < 1$ is observed
across all tested environments and is theoretically expected whenever
the feature manifold has finite intrinsic dimension.
As the spatial window grows, new feature directions are introduced at
a rate proportional to the marginal environmental diversity.  Because
the feature manifold has finite intrinsic dimension, this rate
diminishes, i.e., new locations increasingly fall within
already-observed subspaces.
Experimentally, $\beta$ ranges from $0.40$ to $0.86$ across three environments
(Fig.~\ref{fig:scaling}), with $R^2 > 0.88$.
The fitted exponents are $\beta = 0.64$ (KITTI), $0.86$ (VIVID),
and $0.40$ (HeLiPR), with 95\% confidence intervals strictly
between $0$ and $1$ on all three datasets
(fitting details in Supplementary Sec.~A-6;
model selection in Supplementary Sec.~E).

\subsubsection{Overdetermination of Local Alignment}
\label{sec:cor_local}

Combining Theorem~\ref{thm:identifiability} with the sublinear
scaling law (Eq.~\eqref{eq:scaling}) shows that local alignment is
overdetermined.
A sliding window of $W$ paired observations spanning coverage~$L$
provides $W$ constraints, while the number of data-determined
parameters is $r(2K{-}r{-}1)/2$ with
$r = \lceil a\,L^{\beta}\rceil$.
Because $\beta < 1$, the number of constraints $W$ grows faster
than $r$, so the system becomes increasingly overdetermined as
$L$ increases (Supplementary Sec.~A-6).
Fig.~\ref{fig:scaling} shows that $W$ exceeds $\reff$ across all
tested coverage ranges; in practice, $r \leq W$ holds for
$W \geq 7$.

\subsection{Estimation of the Local Alignment}
\label{sec:graph_registration}
The preceding results assume that matched pairs $(\by_i, \bz_i)$
are available.
In practice, computing~$\bR$ via Procrustes requires
correspondences, yet establishing correspondences via descriptor
matching requires~$\bR$.
As in SLAM, where prediction and update are mutually dependent,
this coupling can be resolved by iterative refinement from a coarse
initial estimate.
Sec.~\ref{sec:method} develops a Soft~EM procedure that
alternates between soft-assigning correspondences within each local
window and updating the rotation estimate.

% =====================================================================
%  IV.  METHOD
% =====================================================================
\section{Training-Free Local Calibration}
\label{sec:method}

The alignment layer is training-free.
It aligns the outputs of any two frozen encoders via closed-form SVD,
without gradient descent or learned parameters.
Because $\bR(\bx)$ is non-stationary, the estimator must be local.
The low effective rank implies that $W_{\mathrm{eff}}{\geq}7$
frames suffice, and anisotropic sensitivity
(Theorem~\ref{thm:perturbation}) tolerates noisy correspondences.

Given a coarse initial position, Algorithm~\ref{alg:sliding}
estimates the local rotation~$\bR$ so that cross-modal descriptors
become directly comparable.
A Soft~EM procedure resolves the mutual dependence between~$\bR$
and the correspondences by
alternating between soft-assigning correspondences (E-step) and
updating the rotation estimate (M-step).
The procedure converges whenever the initialization error
remains within the spatial coherence length of~$\bR(\bx)$
(Supplementary Sec.~F, Table~A2).

\begin{algorithm}[!b]
	\caption{Local Soft Expectation-Maximization}
	\label{alg:sliding}
	\begin{algorithmic}[1]
		\REQUIRE Query $\{(\bz_i, \bp_i)\}_{i=1}^{N}$;
		database $\{(\by_j, \boldsymbol{q}_j)\}_{j=1}^{N_{\mathrm{db}}}$;
		search radius~$r$; spatial prior width~$\sigma$;
		temperature~$\tau$; decay factor~$\lambda$
		\STATE $\mathcal{N}_i \leftarrow
		\{j : \|\bp_i - \boldsymbol{q}_j\| < r\}$ for each $i$
		\hfill\COMMENT{local window}
		\STATE Precompute spatial log-prior:
		$L_{ij} \leftarrow
		-\|\bp_i - \boldsymbol{q}_j\|^2 / (2\sigma^2)$
		for $j \in \mathcal{N}_i$
		\STATE $\bM \leftarrow \boldsymbol{0}_{K\times K}$
		\FOR{$i = 1, \ldots, N$}
		\STATE $\bR \leftarrow \mathrm{Procrustes}(\bM)$
		\STATE $s_j \leftarrow \cos(\bR\,\bz_i,\,\by_j) / \tau + L_{ij}$
		\hfill\COMMENT{E-step}
		\STATE $w_j \leftarrow \mathrm{softmax}_j(s_j)$
		\STATE $\tilde{\by} \leftarrow \sum_j w_j\,\by_j$
		\STATE $\bM \leftarrow \lambda\,\bM + \bz_i\,\tilde{\by}^\top$
		\hfill\COMMENT{M-step}
		\STATE $\bz_{\mathrm{aligned},i} \leftarrow \bz_i\,\bR$
		\STATE $j^*_i \leftarrow \argmax_j
		\cos(\bz_{\mathrm{aligned},i},\,\by_j)$
		\ENDFOR
		\RETURN Retrieval indices $\{j^*_i\}$
	\end{algorithmic}
\end{algorithm}

\subsection{Local Soft Expectation-Maximization}
\label{sec:algorithm}
Given query features $\{(\bz_i, \bp_i)\}$ with coarse positions
and a geotagged database $\{(\by_j, \boldsymbol{q}_j)\}$, the
correspondences between them are unobserved, making this a
latent-variable problem.
The E-step infers soft correspondences from the current rotation
estimate and the spatial prior. The softmax-weighted database
features yield a pseudo co-observation~$\tilde{\by}$ at each
query position.
The M-step updates the cross-covariance~$\bM$ from these pseudo
co-observations and recomputes~$\bR$ via Procrustes
(Algorithm~\ref{alg:sliding}).

\subsubsection{Hyperparameters}
\label{sec:param_analysis}
Algorithm~\ref{alg:sliding} depends on five hyperparameters,
all set once and held fixed across datasets.

\textbf{PCA dimension $K$.}\quad
$K{=}128$ captures $>$91\% of the oracle Procrustes capacity
with diminishing returns beyond this point.

\textbf{Decay factor $\lambda$.}\quad
The effective window size is $W_{\mathrm{eff}} \approx 1/(1{-}\lambda)$.
With $\lambda{=}0.85$ ($W_{\mathrm{eff}} \approx 7$), exponential
forgetting smoothly tracks the non-stationary rotation, avoiding
the rank discontinuities of a hard sliding window.
The averaging operates on~$\bM$, not on~$\bR$ directly, because an
exponential moving average in rotation space loses the singular-value
information that the SVD requires to resolve the signal subspace.

\textbf{Spatial prior width $\sigma$.}\quad
$\sigma$ must be smaller than the coherence length of $\bR(\bx)$,
which scales as $W_{\mathrm{eff}}$ times the inter-frame distance
(typically $25$--$50$\,m for driving datasets).
Performance is robust to moderate increases in $\sigma$
(Supplementary Sec.~I).

\textbf{Temperature $\tau$.}\quad
$\tau$ scales the softmax correspondence weight~$w_j$ in the E-step.
Small $\tau$ sharpens the distribution so that the
highest-scoring database frame dominates~$w_j$,
while large $\tau$ spreads weight across candidates.
$\tau{=}0.1$ is set as a trade-off between the two regimes.

\begin{figure}[!t]
	\centering
	\includegraphics[width=0.82\linewidth]{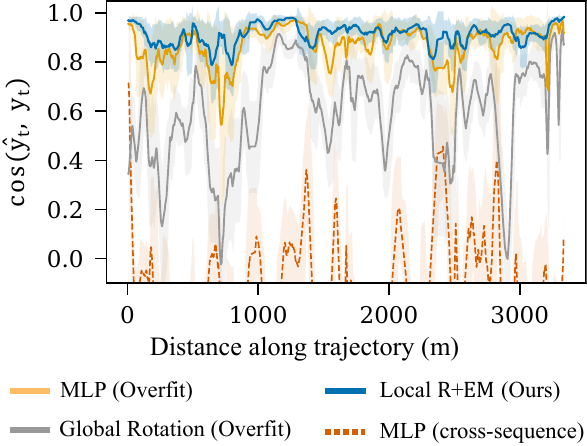}
	\caption{Per-frame cosine similarity
		$\cos(\bR\bz_t, \by_t)$ between the aligned query and the
		ground-truth database descriptor along the VIVID trajectory
		(Gaussian-smoothed, ${\pm}1\sigma$ shading).
		Local Soft~EM (blue) maintains $\cos > 0.8$ throughout.
		A single global $\bR$ fitted to all co-observations
		(gray) oscillates between $0.3$ and $0.8$.
		The MLP adapter~\cite{FreezeAlign2025} trained on
		the same environment (orange) tracks the local estimator
		in trained segments, but degrades
		in a different environment (dashed).}
	\label{fig:coldstart}
\end{figure}

\subsubsection{Estimator Characteristics}
\label{sec:estimator_char}
Three additional properties affect how the estimator
performs in practice.

\textbf{PCA transferability.}\quad
The PCA bases $\bV_z, \bV_y$ are fitted offline on a training split
and applied without modification to all test environments.
Replacing training-set bases with test-set bases changes retrieval
accuracy by less than one percentage point, even across different
datasets and encoder pairs, confirming that the principal subspace
is a property of the descriptor distribution rather than a
dataset-specific artifact.
At test time, only $\bR$ requires online estimation.

\textbf{Sensitivity to $(\sigma,\tau)$.}\quad
A grid sweep over $\sigma \in [10, 80]$\,m and $\tau \in [0.05, 0.5]$
reveals a broad plateau (Supplementary Fig.~A1).
All combinations with $\sigma \leq 30$\,m and $\tau \leq 0.2$
achieve $>$90\% Recall@1 on both KITTI and VIVID.
Performance degrades only at extreme settings.
Beyond $\sigma > 60$\,m the spatial prior becomes uninformative,
and beyond $\tau > 0.3$ the softmax flattens so that the E-step
loses discriminability.
The default pair $(\sigma{=}30\text{\,m},\;\tau{=}0.1)$ lies
well inside this plateau and is used without per-dataset tuning
in all experiments (Supplementary Sec.~I).

% fig_kw_heatmap moved to Appendix~\ref{app:extra_figs}

\textbf{Cold-start behavior.}\quad
Because $\bM$ is initialized to zero, the Procrustes step
is undefined until enough outer products $\bz_i\tilde{\by}^\top$
have accumulated.
The earliest frames also produce noisy E-step assignments
due to the absence of a reliable~$\bR$, creating a brief
initial transient whose length depends on the local
descriptor quality and the inter-frame spacing.

\subsection{Cross-Modal RANSAC}
\label{sec:ransac}

\subsubsection{Prior-Free Setting}
\label{sec:ransac_scenario}
The Soft~EM algorithm operates locally and therefore
requires an initial rotation estimate to enter its
convergence radius.
When a coarse position prior is available, the spatial
window restricts the candidate set and the exponential
averaging on~$\bM$ accumulates a reliable~$\bR$
within $W_{\mathrm{eff}}$ frames.
Otherwise, a hypothesis-and-verify stage can bootstrap
$\bR$ from a small set of frames so that Soft~EM can
take over.
A minimal hypothesis requires only two correspondences,
which suffice to determine~$\bR$ on the signal subspace
(Sec.~\ref{sec:two_pairs}).

\subsubsection{Hypothesis Sampling}
\label{sec:ransac_algo}
Each iteration of Algorithm~\ref{alg:ransac} samples two
query frames and pairs each with a database candidate
drawn from its top-$k$ cosine neighbors.
The resulting cross-covariance is solved by Procrustes,
and the hypothesis is scored by counting inliers whose
$\bR$-aligned cosine similarity exceeds~$\tau_{\mathrm{inlier}}$.
After $T_{\max}$ iterations the best hypothesis is refined
by re-running Procrustes on all its inliers.

\begin{algorithm}[!b]
  \caption{Cross-Modal RANSAC}
  \label{alg:ransac}
  \begin{algorithmic}[1]
    \REQUIRE Query features $\{\bz_i\}_{i=1}^{N}$ at positions
             $\{\bp_i\}$; database $\{(\by_j, \boldsymbol{q}_j)\}$;
             prior center $\hat{\boldsymbol{c}}$; radius~$\rho$
    \STATE $\mathcal{D} \leftarrow
           \{j : \|\boldsymbol{q}_j - \hat{\boldsymbol{c}}\| < \rho\}$
           \COMMENT{candidate selection}
    \FOR{$\mathrm{iter} = 1$ \TO $T_{\max}$}
      \STATE Sample query index $i$; pick database candidate
             $j \in \mathcal{D}$ from top-$k$ cosine
      \STATE Sample $k \neq i$; pick $l \in \mathcal{D}$
             from top-$k$ cosine
      \STATE $\bR \leftarrow\mathrm{Procrustes}~\!\bigl(
             [\bz_i;\bz_k],\,[\by_j;\by_l]\bigr)$
      \STATE $j^*_m \leftarrow \arg\max_{j \in \mathcal{D}}\cos(\bR\bz_m,\,\by_j)$ for each query $m$
      \STATE Score: $|\{m : \cos(\bR\bz_m, \by_{j^*_m}) > \tau_{\mathrm{inlier}}\}|$
      \STATE Keep best $(\bR, \mathrm{score})$
    \ENDFOR
    \STATE Refine: iterative Procrustes on all inliers of $\bR^*$
    \RETURN Best rotation $\bR^*$; inlier correspondences
  \end{algorithmic}
\end{algorithm}

\subsubsection{Two-Point Random Sample Consensus}
\label{sec:two_pairs}

Although $\SO(128)$ has $K(K{-}1)/2 = 8128$ free parameters,
two correspondences
$\bM = \bz_1\by_1^\top + \bz_2\by_2^\top$ already suffice
to bootstrap~$\bR$.
Each correspondence contributes one rank to~$\bM$, so two
from distinct locations give $\rank(\bM) = 2$.
By Theorem~\ref{thm:identifiability}, $r{=}2$ determines
$r(2K{-}r{-}1)/2 = 2(2{\cdot}128{-}3)/2 = 253$ signal DOF
of~$\bR$.
The remaining $8128 - 253 = 7875$ DOF lie in the null space
of~$\bM$ and are fixed solely by the orthogonality constraint
$\bR^\top\bR = \boldsymbol{I}$
(Theorem~\ref{thm:identifiability}).
By Theorem~\ref{thm:perturbation}, these components have
negligible effect on retrieval because the descriptors
carry near-zero energy along the null directions.
When $\reff \leq 2$, as on KITTI ($\reff{=}1.5$), two
correspondences fully capture the signal subspace.
For higher effective ranks the rank-2 hypothesis
covers only a subset of the signal directions, but
this partial estimate already suffices to bootstrap
Soft~EM, which then refines~$\bR$ with additional
frames.
The RANSAC inner loop, comprising a Procrustes solve from a
minimal sample and inlier scoring, is computationally inexpensive.
The SVD of a $K{\times}2$ matrix costs $\mathcal{O}(K)$, and
verifying $N$ inliers costs $\mathcal{O}(NK)$.
With $\sigma{=}30$\,m the candidate pool
$|\mathcal{D}|$ typically contains $20$--$50$ database frames,
and convergence is reached within tens of iterations
(Sec.~\ref{sec:exp_ransac}).

% ────────────────────────────────────────────────────────────────────
%  IV-D.  Cross-Modal Odometry Pipeline (implementation)
% ────────────────────────────────────────────────────────────────────
\subsection{Cross-Modal Retrieval Pipeline}
\label{sec:odometry}

The alignment converts a query descriptor into a
database-compatible representation, enabling cross-modal
retrieval between arbitrary encoder pairs.
We embed this retrieval in a GTSAM~\cite{GTSAM} pose-graph
to evaluate end-to-end positioning.
The setting is \emph{asymmetric deployment}.
A geo-tagged database
$\{(\by_j, \boldsymbol{q}_j)\}$ is built offline from one
modality, and at query time a different modality provides
descriptors $\{\bz_i\}$ together with odometry and a coarse
position prior.

\label{sec:two_pass}

\textbf{Pass~1 (forward, Soft~EM alignment + retrieval).}\quad
The query sequence is processed frame by frame via
Algorithm~\ref{alg:sliding}.
For each frame~$t$, the Soft~EM procedure estimates a local
rotation~$\bR(t)$ using the predicted position~$\hat{\bp}(t)$
as the spatial prior center.
The aligned descriptor $\bR(t)\,\bz(t)$ is then matched against
the database to find the nearest neighbor
$j^*(t) = \argmax_j \cos(\bR(t)\,\bz(t),\,\by_j)$,
and its geo-tagged position is stored as a retrieval measurement.
Importantly, the predicted position is not updated during
this pass, so that retrieval errors do not feed back into the
spatial prior of subsequent frames.

\textbf{Pass~2 (batch optimization).}\quad
Once all retrieval measurements are collected, a GTSAM
factor graph fuses them with the odometry and the initial
position prior.
Each retrieval measurement enters as a unary factor on
the corresponding $\mathrm{Point2}(x,y)$ node, wrapped in
a Huber robust
kernel to suppress outliers.
Because the solver optimizes all frames jointly,
corrections propagate bidirectionally across the
entire trajectory.
When the encoder pair includes a camera, retrieval error is
anisotropic: lateral displacement is well-resolved, but
along-track displacement is not, because 2D projection
discards depth.
We account for this by assigning a heading-aligned
elliptical noise model to each retrieval factor, with
tight lateral variance and loose along-track variance,
letting odometry resolve the depth ambiguity.

% =====================================================================
%  V.  EXPERIMENTS
% =====================================================================
\section{Experiments}
\label{sec:experiments}
To verify that the training-free alignment generalizes
across encoder pairs, sensor configurations, and
environments, we organize experiments along the four transmodal axes
introduced in Sec.~\ref{sec:introduction},
namely cross-modal (Table~\ref{tab:crossmodal}), cross-device
(Table~\ref{tab:cross_sensor}), cross-representation
(Table~\ref{tab:crossrep}), and cross-backbone
(Tables~\ref{tab:kitti_pairs},~\ref{tab:city_crossseq}).

\begin{table}[!b]
	\centering
	\caption{Evaluation datasets and their sensor suites.
		The three datasets span suburban, campus, and urban
		environments with four modalities (LiDAR, RGB, thermal
		infrared, event).}
	\label{tab:datasets}
	\small
	\begin{tabular}{@{}lll@{}}
		\toprule
		Dataset & Environment & Sensors \\
		\midrule
		KITTI-360~\cite{KITTI360} & Suburban
		& \makecell[tl]{Velodyne (HDL-64E)\\PointGrey (stereo)} \\
		\cmidrule(lr){1-3}
		VIVID~\cite{VIVID} & Campus
		& \makecell[tl]{Ouster (OS1-64)\\PointGrey (fisheye)\\FLIR A65 (LWIR)\\DVXplorer (event)} \\
		\cmidrule(lr){1-3}
		HeLiPR~\cite{HeLiPR} & Urban
		& \makecell[tl]{Ouster (OS2-128)\\Velodyne (VLP-16)\\Livox (Avia)} \\
		\bottomrule
	\end{tabular}
\end{table}

\subsection{Datasets and Setup}
\label{sec:datasets}

\textbf{Datasets.}\quad
Table~\ref{tab:datasets} lists the three evaluation datasets.
KITTI-360~\cite{KITTI360} provides a dense suburban sequence
(seq\,0010) with co-registered LiDAR and camera frames.
VIVID~\cite{VIVID} offers two sequences (Campus and City) with
four co-registered modalities.
The cross-sequence protocol
(day2 query, day1 database) tests generalization across sessions.
HeLiPR~\cite{HeLiPR} contains three heterogeneous LiDARs with
different scan patterns (uniform, sparse, non-repetitive),
enabling cross-device evaluation across scan patterns.

\textbf{Encoders.}\quad
All encoders are frozen at test time.
\textsc{Rosetta} never modifies their weights.
For image features we use publicly available foundation
models, with DINOv3-ViT-L/16~\cite{DINOv3} (300\,M parameters)
as the default and DINOv2-ViT-L and
DINOv2-ViT-B~\cite{DINOv2} in cross-backbone experiments.
For point-cloud features we use off-the-shelf encoders
such as MinkLoc3Dv2~\cite{MinkLoc3Dv2} and
PointNetVLAD~\cite{PointNetVLAD}, along with a
custom-trained point cloud encoder (PCE, 15\,M parameters)
that allows controlled experiments on training
configuration and descriptor dimension.
PCE is a
five-stage spconv~\cite{spconv} network with strided
$2{\times}$ downsampling, global max-pooling, and a GELU
projection head producing a 1024-D descriptor.
PCE is trained with the Barlow Twins loss for 30 epochs
using AdamW ($\mathrm{lr}{=}10^{-3}$) on the eight
non-test KITTI-360 sequences
(0000, 0002--0007, 0009, totaling ${\sim}67$k frames)
and validated on the held-out test sequence~0010.

\textbf{Comparison protocol.}\quad
Joint cross-modal methods such as VXP~\cite{VXP} require
paired training data and modify the encoder weights,
whereas \textsc{Rosetta} operates post-hoc on frozen
encoders with no paired supervision
(Table~\ref{tab:setting}).
A direct accuracy comparison would therefore mix
alignment quality with encoder quality.
Instead, we fix the encoders and compare alignment
strategies that share the same frozen representations,
namely overfitting Procrustes (oracle upper bound fitted on
all test co-observations) and global Procrustes
(fitted on a training split, applied to validation).

\begin{table}[!t]
	\centering
	\caption{Problem setting comparison. Unlike joint or
		supervised methods, \textsc{Rosetta} aligns independently
		trained, frozen encoders without any paired training data.}
	\label{tab:setting}
	\small
	\begin{tabular}{@{}lccc@{}}
		\toprule
		Method & \makecell{Super-\\vision} & \makecell{Learned\\module} & \makecell{Zero-shot\\pairing} \\
		\midrule
		VXP~\cite{VXP}             & Paired & Both enc. & -- \\
		(LC)$^2$~\cite{LC2}        & Paired & Both enc. & -- \\
		MLP~\cite{FreezeAlign2025} & Paired & Projector & -- \\
		\textbf{\textsc{Rosetta}}  & \textbf{None} & \textbf{None} & \textbf{\checkmark} \\
		\bottomrule
	\end{tabular}
\end{table}

\textbf{Baselines.}\quad
Concretely, we compare five alignment strategies on the
same frozen representations.
(i)~\emph{Raw}, no alignment, where the best-trained model
(Barlow Twins + GELU) is used directly.
(ii)~\emph{Overfitting Procrustes$^*$}, a single $\bR$
fitted on all test-set co-observations, serving as an
oracle upper bound on any constant-$\bR$ method.
(iii)~\emph{MLP Adapter}~\cite{FreezeAlign2025}, a 2-layer
MLP projector ($128$-$128$-$128$, ReLU) trained on paired
co-observations with MSE$+$cosine loss (500 epochs, AdamW),
evaluated both same-environment and cross-environment.
(iv)~\emph{Global Procrustes}, where $\bR$ is computed on a
training split and applied to the validation split,
testing whether a single alignment transfers across space.
(v)~\emph{Local Soft~EM (Ours)} (Algorithm~\ref{alg:sliding})
with $K{=}128$, $\tau{=}0.1$, $\lambda{=}0.85$, and
$\sigma{=}15$\,m, chosen well below the coherence
length of $\bR(\bx)$ (${\sim}35$\,m,
Sec.~\ref{sec:param_analysis});
requires only a coarse position prior, no co-observation.
Unless noted otherwise, all methods in Table~\ref{tab:main}
share the same frozen representations,
PCE / DINOv3 for KITTI-360 and VIVID
(LiDAR-to-camera), and two identical PCE instances
with Ouster versus Velodyne inputs for HeLiPR
(cross-device).
The same encoder configuration applies to all
subsequent experiments unless explicitly stated.

\textbf{Test environment.}\quad
All experiments were conducted on a single NVIDIA RTX 4070 Ti SUPER
(16\,GB VRAM) and AMD Ryzen 9 5900XT CPU (single thread).
Timing analysis is reported in Sec.~\ref{sec:exp_compute}.
Source code and pre-extracted features will be released upon
acceptance.

\begin{table}[!t]
	\centering
	\caption{Cross-modal retrieval (R@1@25\,m, global NN).
		$^*$Requires test co-observations.
		$^\dagger$Cross-sequence.
		$^\ddagger$Pretrained~\cite{LC2}.
		$^\S$Fine-tuned per dataset.
		$^\P$MLP trained on target environment.}
	\label{tab:main}
	\footnotesize
	\begin{tabular}{@{}lccc@{}}
		\toprule
		Method & KITTI & VIVID$^\dagger$ & HeLiPR \\
		\midrule
		(LC)$^2$ pretrained$^\ddagger$~\cite{LC2} & 0.025 & 0.037 & 0.031 \\
		(LC)$^2$ fine-tuned$^\S$ & 0.509 & 0.505 & 0.652 \\
		Raw (no alignment) & 0.305 & 0.058 & 0.014 \\
		MLP Adapter (cross-env)   & 0.023 & 0.061 & 0.032 \\
		Overfitting Procrustes$^*$ & 0.998 & 0.857 & 0.649 \\
		MLP Adapter (same-env)$^\P$ & 0.988 & 0.901 & 0.830 \\
		\midrule
		\textbf{Ours}
		& \textbf{0.944} & \textbf{0.933} & \textbf{0.681} \\
		\bottomrule
	\end{tabular}
\end{table}

\subsection{Alignment Comparison}
\label{sec:exp_main}

Table~\ref{tab:main} compares the five alignment strategies
across three datasets.
Without alignment, cross-modal retrieval ranges from
$0.014$ on HeLiPR to $0.305$ on KITTI, consistent with
the non-stationarity ceiling
(Supplementary Sec.~B).
The jointly trained (LC)$^2$~\cite{LC2} pretrained model scores
$0.025$, $0.037$, and $0.031$ across the three datasets.
Per-dataset fine-tuning raises KITTI to $0.509$,
VIVID to $0.505$, and HeLiPR to $0.652$,
but remains substantially below \textsc{Rosetta} on KITTI and VIVID
while requiring paired per-dataset training data.
The overfitting Procrustes$^*$ reaches $0.998$ on KITTI,
where scene variation is limited, but degrades to
$0.857$ on VIVID and $0.649$ on HeLiPR as
non-stationarity grows ($\reff{=}9.9$).
As spatial diversity grows, a single global rotation
becomes insufficient.
The same limitation bounds the global Procrustes
post-processing used in recent embedding-interoperability
methods~\cite{Moschella2023Relative,Maystre2025Procrustes,Vejendla2025DriftAdapter}.

\textbf{MLP adapter.}\quad
The supervised MLP adapter~\cite{FreezeAlign2025} reaches
$0.988$ on KITTI, $0.901$ on VIVID, and $0.830$ on HeLiPR
when trained and tested on the same environment, exceeding
the overfitting Procrustes on all three datasets.
However, cross-environment transfer fails entirely.
A KITTI-trained adapter applied to VIVID scores $0.061$ and
to HeLiPR $0.032$.
Conversely, a VIVID-trained adapter
on KITTI drops to $0.023$, all indistinguishable from the
raw cross-modal baseline.

\textbf{Ours.}\quad
Without any co-observation, Ours achieves $0.944$
on KITTI, $0.933$ on VIVID (cross-sequence), and $0.681$
on HeLiPR.
On VIVID, where $\reff{=}5.1$, Ours surpasses both
the overfitting Procrustes ($0.857$) and the same-env
MLP adapter ($0.901$), because the local estimator tracks
position-dependent rotations that a global mapping
averages away.
On KITTI ($\reff{=}1.5$), non-stationarity is mild and
both global and local approaches perform well
($0.998$ vs.\ $0.944$).
On HeLiPR ($\reff{=}9.9$), the gap narrows
($0.681$ vs.\ $0.649$).
Here, the larger inter-frame distance (5.9\,m) provides
fewer co-observations within the coherence length,
limiting the rank that the local estimator can resolve
(Supplementary Sec.~A-6).
Ours also surpasses the jointly trained
(LC)$^2$ ($0.505$ on VIVID).
Each frozen encoder learns a stronger representation
independently, and our method recovers the inter-encoder
relationship post hoc.\label{rem:frozen_better}
Qualitative examples are shown in
Fig.~\ref{fig:qualitative}.

\begin{table}[!t]
	\centering
	\caption{Cross-device retrieval on HeLiPR, Axis~(ii)
		(R@1@25\,m, global NN). $^*$Requires test co-observations.
		$^\dagger$MinkLoc3Dv2 (256-D; separate PCA).}
	\label{tab:cross_sensor}
	\footnotesize
	\begin{tabular}{llcc}
		\toprule
		Encoder & Sensor pair & Overfitting$^*$ & \textbf{Ours} \\
		\midrule
		\multirow{5}{*}{PCE}
		& Ouster $\to$ Ouster (same-sensor) & \multicolumn{2}{c}{0.967} \\
		\cmidrule(lr){2-4}
		& Ouster $\to$ Velodyne (DCC04)    & 0.649 & \textbf{0.681} \\
		& Ouster $\to$ Avia (DCC04)        & 0.652 & \textbf{0.695} \\
		& Ouster $\to$ Velodyne (KAIST04)  & 0.675 & \textbf{0.732} \\
		& Velodyne $\to$ Ouster (DCC04)    & {0.691} & 0.680 \\
		\midrule
		\multirow{5}{*}{MinkLoc$^\dagger$}
		& Ouster $\to$ Ouster (same-sensor) & \multicolumn{2}{c}{0.998} \\
		\cmidrule(lr){2-4}
		& Ouster $\to$ Velodyne (DCC04)    & 0.959 & 0.876 \\
		& Ouster $\to$ Avia (DCC04)        & 0.780 & \textbf{0.838} \\
		& Ouster $\to$ Velodyne (KAIST04)  & {0.906} & 0.809 \\
		& Velodyne $\to$ Ouster (DCC04)    & {0.930} & 0.893 \\
		\bottomrule
	\end{tabular}
\end{table}

\begin{figure}[!b]
	\centering
	\includegraphics[width=\linewidth]{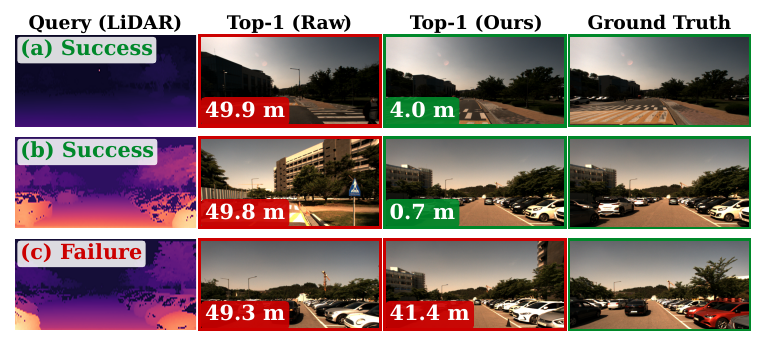}
	\caption{Qualitative cross-modal retrieval on VIVID
		(day2 LiDAR query $\to$ day1 camera database).
		Each row shows the query, the top-1 match without and with
		our alignment, and the ground truth.
		\textbf{(a,\,b)}~Our alignment corrects a wrong raw match.
		\textbf{(c)}~Failure: perceptual aliasing between visually
		similar scenes persists after alignment.}
	\label{fig:qualitative}
\end{figure}

\subsection{Cross-Device Generality}
\label{sec:exp_cross_sensor}

\begin{table}[!t]
	\centering
	\caption{Cross-modal retrieval on VIVID, Axis~(i)
		(R@1@25\,m, global NN).
		$^*$Requires test co-observations.
		$^\dagger$Cross-sequence.}
	\label{tab:crossmodal}
	\small
	\setlength{\tabcolsep}{4pt}
	\begin{tabular}{l cc}
		\toprule
		Pair & Overfitting$^*$ & \textbf{Ours} \\
		\midrule
		$\mathcal{T}{\to}\mathcal{R}^\dagger$ & 0.776 & \textbf{0.858} \\
		$\mathcal{R}{\to}\mathcal{T}^\dagger$ & 0.873 & \textbf{0.886} \\
		$\mathcal{E}{\to}\mathcal{R}^\dagger$ & 0.769 & 0.737 \\
		$\mathcal{R}{\to}\mathcal{E}^\dagger$ & 0.751 & \textbf{0.756} \\
		$\mathcal{T}{\to}\mathcal{E}^\dagger$ & 0.609 & \textbf{0.744} \\
		$\mathcal{E}{\to}\mathcal{T}^\dagger$ & 0.739 & 0.734 \\
		\bottomrule
	\end{tabular}
\end{table}

\begin{table}[!b]
	\centering
	\caption{Cross-representation retrieval on VIVID Campus,
		Axes~(iii)\,\&\,(iv) (R@1@25\,m, global NN).
		$^*$Requires test co-observations.
		$^\dagger$MinkLoc3Dv2 (256-D; separate PCA).}
	\label{tab:crossrep}
	\footnotesize
	\setlength{\tabcolsep}{4pt}
	\begin{tabular}{ll cc}
		\toprule
		Pair & Encoders & Overfitting$^*$ & \textbf{Ours} \\
		\midrule
		\multicolumn{4}{l}{\textit{Cross-sequence (day2$\to$day1):}} \\
		$\mathcal{L}{\to}\mathcal{R}$ & PCE / DINOv3          & 0.859 & \textbf{0.909} \\
		$\mathcal{R}{\to}\mathcal{L}$ & DINOv3 / PCE          & {0.711} & 0.677 \\
		$\mathcal{L}{\to}\mathcal{R}$ & MinkLoc$^\dagger$ / DINOv3 & 0.905 & \textbf{0.946} \\
		$\mathcal{R}{\to}\mathcal{L}$ & DINOv3 / MinkLoc$^\dagger$ & 0.739 & 0.736 \\
		$\mathcal{E}{\to}\mathcal{R}$ & Event DINOv3 / DINOv3 & 0.778 & 0.735 \\
		$\mathcal{R}{\to}\mathcal{E}$ & DINOv3 / Event DINOv3 & 0.749 & 0.770 \\
		\midrule
		\multicolumn{4}{l}{\textit{Within-sequence (day1, self-exclusion):}} \\
		$\mathcal{L}{\to}\mathcal{R}$ & PNV / DINOv3          & 0.631 & \textbf{0.865} \\
		$\mathcal{R}{\to}\mathcal{L}$ & DINOv3 / PNV          & 0.720 & \textbf{0.721} \\
		$\mathcal{E}{\to}\mathcal{R}$ & EventVLAD / DINOv3    & 0.512 & \textbf{0.653} \\
		$\mathcal{R}{\to}\mathcal{E}$ & DINOv3 / EventVLAD    & 0.608 & \textbf{0.667} \\
		\bottomrule
	\end{tabular}
\end{table}
To evaluate whether our method generalizes beyond
cross-modal (LiDAR--camera) to cross-device (LiDAR--LiDAR) settings,
experiments are conducted on HeLiPR~\cite{HeLiPR}, which provides
synchronized data from three LiDAR sensors with different scan
patterns, namely Ouster OS2-128 (uniform), Velodyne VLP-16 (sparse), and
Livox Avia (non-repetitive).
All three encoders are identical PCE instances trained
on Ouster data, so the alignment gap arises solely from the difference
in scan geometry.
Raw cross-device retrieval drops to $0.01$--$0.03$ despite $0.967$
same-sensor performance (Table~\ref{tab:cross_sensor}).
Table~\ref{tab:cross_sensor} shows that Ours matches or
exceeds the overfitting Procrustes on three of four PCE pairs
($0.68$--$0.73$ vs.\ $0.65$--$0.69$), with the remaining pair
within $0.01$, confirming that the local estimator attains
overfitting-Procrustes-level alignment without any
co-observation.
The gauge transformation model therefore extends to differently
conditioned encoders, not only LiDAR--camera combinations.
The remaining gap to same-sensor performance ($0.967$) is not
an alignment failure but an encoder generalization issue.
PCE is trained exclusively on Ouster scans, so its descriptors
degrade on the sparser Velodyne and non-repetitive Avia inputs.
Repeating the experiment with MinkLoc3Dv2, which is trained on
diverse scan patterns, confirms this interpretation.
Same-sensor performance rises to $0.998$, and cross-sensor
retrieval with Ours reaches $0.81$--$0.89$
(vs.\ $0.68$--$0.73$ for PCE), halving the gap to same-sensor
and confirming that the shortfall is an encoder generalization
issue, not an alignment failure.

\subsection{Transmodal Generality}
\label{sec:exp_multimodal}
The preceding subsections established cross-device generality on HeLiPR.
This subsection evaluates the remaining three axes (cross-modal,
cross-representation, and cross-backbone) on the VIVID campus dataset,
which provides synchronized LiDAR (Ouster OS1-64), RGB (PointGrey fisheye),
thermal infrared (FLIR A65, 14-bit), and event camera (DVXplorer) data.
DINOv3-ViT-L/16 is applied to all three visual modalities without
modification (thermal images are quantized to 8-bit grayscale and
replicated to three channels, and event frames are EventVLAD
reconstructions).
Tables~\ref{tab:crossmodal} and~\ref{tab:crossrep} report
cross-modal and cross-representation retrieval results, respectively,
using the cross-sequence protocol (day2 query, day1 database).
\paragraph{Cross-modal}
Table~\ref{tab:crossmodal} shows that
when DINOv3 is applied to all visual modalities,
Ours achieves $0.73$--$0.89$ R@1@25\,m across all six directions,
exceeding the overfitting Procrustes on four of six pairs.
Thermal--RGB ($0.858$/$0.886$) is the strongest pair,
consistent with DINOv3 encoding similar structural
information from thermal and RGB inputs.
\paragraph{Cross-representation}
Table~\ref{tab:crossrep} shows that
when query and database use independently trained encoders,
raw retrieval is near zero, yet
Ours recovers $0.65$--$0.95$.
MinkLoc3Dv2~\cite{MinkLoc3Dv2}, a unimodal LiDAR
encoder with no joint cross-modal training, paired with DINOv3
achieves $0.946$ ($\mathcal{L}{\to}\mathcal{R}$) with our method,
the highest among all encoder pairs.
PointNetVLAD~\cite{PointNetVLAD} paired with DINOv3
achieves $0.865$,
demonstrating that even legacy encoders benefit from the alignment.

\paragraph{Directional asymmetry}
In cross-representation retrieval,
$\mathcal{L}{\to}\mathcal{R}$ outperforms
$\mathcal{R}{\to}\mathcal{L}$ by $0.14$--$0.23$
across all three LiDAR--camera encoder pairs
in Table~\ref{tab:crossrep}.
This asymmetry arises because the alignment estimates
$\bR$ from the database-side cross-covariance
(Sec.~\ref{sec:algorithm}).
When the database encoder (DINOv3) has higher
discriminability, the cross-covariance matrix $\bM$ is
better conditioned, yielding more accurate rotation estimates.
This confirms that alignment quality is upper-bounded by
the discriminability of the underlying representations,
particularly on the database side.

\begin{table}[!t]
  \centering
  \caption{Encoder-pair sweep on KITTI-360
  (R@1@25\,m, global NN).
  $^*$Requires test co-observations.
  $^\times$MLP trained on VIVID, transferred to KITTI.}
  \label{tab:kitti_pairs}
  \footnotesize
  \setlength{\tabcolsep}{3.5pt}
  \begin{tabular}{ll ccc}
    \toprule
    Pair & Encoders & Overfitting$^*$ & MLP$^\times$ & \textbf{Ours} \\
    \midrule
    \multicolumn{5}{l}{\textit{Cross-representation (Axis~iii):}} \\
    $\mathcal{L}{\to}\mathcal{R}$ & PCE / DINOv3        & 0.963 & 0.013 & 0.945 \\
    $\mathcal{R}{\to}\mathcal{L}$ & DINOv3 / PCE        & 0.868 & 0.007 & \textbf{0.966} \\
    $\mathcal{L}{\to}\mathcal{R}$ & PCE / DINOv2-ViT-L  & 0.965 & 0.002 & \textbf{0.965 }\\
    $\mathcal{R}{\to}\mathcal{L}$ & DINOv2-ViT-L / PCE  & 0.922 & 0.017 & \textbf{0.960} \\
    $\mathcal{L}{\to}\mathcal{R}$ & PCE / DINOv2-ViT-B  & 0.960 & 0.018 & \textbf{0.969} \\
    $\mathcal{R}{\to}\mathcal{L}$ & DINOv2-ViT-B / PCE  & 0.904 & 0.009 & \textbf{0.976} \\
    \midrule
    \multicolumn{5}{l}{\textit{Cross-backbone (Axis~iv):}} \\
    $\mathcal{R}{\to}\mathcal{R}$ & DINOv3 / DINOv2-ViT-L     & 0.984 & 0.016 & \textbf{0.989} \\
    $\mathcal{R}{\to}\mathcal{R}$ & DINOv2-ViT-L / DINOv3     & 0.995 & 0.004 & 0.987 \\
    $\mathcal{R}{\to}\mathcal{R}$ & DINOv3 / DINOv2-ViT-B     & 0.987 & 0.029 & \textbf{0.992} \\
    $\mathcal{R}{\to}\mathcal{R}$ & DINOv2-ViT-B / DINOv3     & 0.989 & 0.046 & 0.985 \\
    $\mathcal{R}{\to}\mathcal{R}$ & DINOv2-ViT-L / DINOv2-ViT-B & 0.998 & 0.001 & 0.997 \\
    $\mathcal{R}{\to}\mathcal{R}$ & DINOv2-ViT-B / DINOv2-ViT-L & 1.000 & 0.023 & 0.991 \\
    \bottomrule
  \end{tabular}
\end{table}

\paragraph{Cross-backbone universality}
Tables~\ref{tab:kitti_pairs} and~\ref{tab:city_crossseq} extend
the evaluation to all pairwise encoder combinations on KITTI-360
and VIVID City, covering both cross-representation (Axis~iii) and
cross-backbone (Axis~iv) pairs.
On KITTI-360, Ours achieves $0.94$--$0.99$ on cross-representation
and $0.98$--$0.99$ on cross-backbone pairs.
On VIVID City, where stronger non-stationarity degrades the
overfitting Procrustes, Ours outperforms it by $0.06$--$0.24$
across all pairs.
The rotation dynamics underlying this invariance do not depend
on a shared training procedure.
Replacing DINOv3-ViT-L with DINOv2-ViT-L~\cite{DINOv2}, a backbone
never seen during any training stage, yields nearly identical
frame-to-frame rotation norms
($\lVert\Delta\bR\rVert_F{=}14.47$ versus $14.46$)
and comparable autocorrelation ($\rho_1{=}0.182$).
The MLP$^\times$ columns in both tables confirm that a supervised
adapter trained on one environment collapses to $0.00$--$0.05$ when
transferred to another, whereas Ours requires no retraining.

\paragraph{Cross-modal RANSAC}
\label{sec:exp_ransac}
With a position prior of $\sigma{=}30$\,m and $N_q{=}3$ query frames,
RANSAC achieves perfect retrieval on KITTI (raw $0.305$,
${\sim}82$ iterations).
The method is encoder-agnostic.
On VIVID, replacing the encoder pair
yields identical RANSAC performance, because the spatial consistency
constraint dominates over feature quality when the GPS prior is tight.
Detailed RANSAC results across datasets, including the dependence on
$N_q$ and $\sigma$, are reported in Supplementary Sec.~H.

\paragraph{Scaling law verification}
\label{sec:exp_scaling}
Sec.~\ref{sec:thm_scaling} predicted that the effective rank
grows sublinearly with window length, $\reff = a\,L^{\beta}$
with $\beta < 1$.
Fig.~\ref{fig:scaling} validates this prediction.
The power law holds within
each dataset ($R^2 > 0.88$, cross-validated $R^2 > 0.80$),
with $\beta$ ranging from $0.40$ (HeLiPR, repetitive urban
facades) to $0.86$ (VIVID, diverse campus).
Cross-dataset prediction yields $R^2 < 0.5$, confirming that
$(a, \beta)$ are deployment-specific, but the functional form
$\beta < 1$ holds across all tested environments, so local alignment
is consistently better conditioned than global.

\begin{table}[!t]
	\centering
	\caption{Cross-sequence retrieval on VIVID City
		(day2 query, day1 database; R@1@25\,m, global NN).
		$^*$Requires test co-observations.
		$^\times$MLP trained on KITTI, transferred to VIVID City.
		$^\dagger$MinkLoc3Dv2 (256-D; separate PCA).}
	\label{tab:city_crossseq}
	\footnotesize
	\setlength{\tabcolsep}{3.5pt}
	\begin{tabular}{ll ccc}
		\toprule
		Pair & Encoders & Overfitting$^*$ & MLP$^\times$ & \textbf{Ours} \\
		\midrule
		\multicolumn{5}{l}{\textit{Cross-representation (Axis~iii):}} \\
		$\mathcal{L}{\to}\mathcal{R}$ & PCE / DINOv3            & 0.611 & 0.007 & \textbf{0.758} \\
		$\mathcal{R}{\to}\mathcal{L}$ & DINOv3 / PCE            & 0.612 & 0.007 & \textbf{0.856} \\
		$\mathcal{L}{\to}\mathcal{R}$ & MinkLoc$^\dagger$ / DINOv3 & 0.762 & 0.005 & \textbf{0.839} \\
		$\mathcal{R}{\to}\mathcal{L}$ & DINOv3 / MinkLoc$^\dagger$ & 0.686 & 0.005 & \textbf{0.821} \\
		$\mathcal{L}{\to}\mathcal{R}$ & PCE / DINOv2-ViT-L      & 0.695 & 0.005 & \textbf{0.795} \\
		$\mathcal{R}{\to}\mathcal{L}$ & DINOv2-ViT-L / PCE      & 0.698 & 0.006 & \textbf{0.785} \\
		$\mathcal{L}{\to}\mathcal{R}$ & MinkLoc$^\dagger$ / DINOv2-ViT-L & 0.748 & 0.003 & \textbf{0.881} \\
		$\mathcal{R}{\to}\mathcal{L}$ & DINOv2-ViT-L / MinkLoc$^\dagger$ & 0.744 & 0.005 & \textbf{0.866} \\
		\midrule
		\multicolumn{5}{l}{\textit{Cross-backbone (Axis~iv):}} \\
		$\mathcal{R}{\to}\mathcal{R}$ & DINOv3 / DINOv2-ViT-L   & 0.786 & 0.003 & \textbf{0.863} \\
		$\mathcal{R}{\to}\mathcal{R}$ & DINOv2-ViT-L / DINOv3   & 0.853 & 0.014 & \textbf{0.922} \\
		\bottomrule
	\end{tabular}
\end{table}

\subsection{Corroboration of Non-Stationarity}
\label{sec:exp_negative}

The non-stationarity analysis of Sec.~\ref{sec:nonstationarity}
predicts that any method assuming a fixed or slowly varying $\bR$
will fail when applied beyond the coherence length.
We test four such approaches on KITTI-360
using the PCE / DINOv3 encoder pair
(Table~\ref{tab:negative}).
Each fails for a reason directly traceable to the
theoretical analysis.

\begin{table}[!t]
	\centering
	\caption{Failed approaches on KITTI-360
		(Inlier({$<$}25\,m)).
		Each method attempts to learn or predict $\bR(x)$
		during training.}
	\label{tab:negative}
	\footnotesize
	\begin{tabular}{@{}lrl@{}}
		\toprule
		Approach & Inlier({$<$}25\,m) & Reason of failure\\
		\midrule
		Diff.\ Procrustes adapter   & 0.07--0.09  & SVD backward ill-cond. \\
		Rotation prediction MLP     & 0.15--0.20 & White-noise $\bOm$ \\
		$\bR$-consistency reg.      & 0.076   & Non-stationary process \\
		Cross-traversal             & 0.418  & Backbone repeatability \\
		\bottomrule
	\end{tabular}
\end{table}
\paragraph{Differentiable Procrustes adapter ($0.07$--$0.09$)}
A learnable SVD layer is inserted between the encoders, and the entire
pipeline is trained end to end with the retrieval loss.
The SVD backward pass requires division by $\sigma_i - \sigma_j$,
which is ill-conditioned when singular values cluster, as is the case
in the low-$\reff$ regime.
Replacing the full SVD with \texttt{svd\_lowrank} at $r{=}16$
stabilizes the gradients but limits the adapter to rank-16
corrections, which are insufficient for the full-rank $\bOm$
dynamics.
Further details are provided in Supplementary Sec.~D.

\paragraph{Rotation prediction MLP ($0.15$--$0.20$)}
An MLP is trained to predict the local $\bR(t)$ from a concatenation
of query and database features.
Because $\bOm$ is white noise (Sec.~\ref{thm:nonstationary}), the
optimal predictor from features alone is the unconditional mean,
which reduces to the global Procrustes solution.

\paragraph{$\bR$-consistency regularization ($0.076$)}
Adding $\lambda\|\bR(t){-}\bR(t{-}1)\|_F^2$ to the training loss
encourages temporal smoothness.
This assumption conflicts with the observed dynamics, in which the
alignment changes abruptly between consecutive frames
(autocorrelation $= -0.018$).

\paragraph{Cross-traversal alignment ($0.418$)}
Features from a different traversal of the same route are used as
pseudo ground truth for alignment.
This approach fails because foundation-model features have limited
viewpoint repeatability.
The same location produces different descriptors when observed from
different trajectories, times, or weather conditions.

\subsection{Computational Cost}
\label{sec:exp_compute}
\label{sec:cost}

The alignment pipeline, comprising PCA projection ($0.1$\,ms),
Soft~EM iteration with $128{\times}128$ SVD ($1.1$\,ms),
and cosine retrieval ($0.1$\,ms),
runs on CPU in $1.2$\,ms per query.
Feature extraction dominates, at $0.3$\,s for DINOv3 and $0.050$\,s
for PCE on an RTX 4070 Ti SUPER.
At deployment the camera encoder (DINOv3) is evaluated only once
during map construction.
Each live query requires only the LiDAR
encoder ($50$\,ms) plus alignment and retrieval ($1.2$\,ms),
totaling ${\sim}51$\,ms, compatible with $10$\,Hz LiDAR operation.
By contrast, cross-modal methods that learn shared
embeddings (e.g., VXP at ${\sim}200$\,ms per LiDAR frame)
incur comparable extraction costs and additionally require
paired training, which our method eliminates entirely.
The alignment overhead is dominated by the $\mathcal{O}(K^3)$ SVD,
which can be further reduced via randomized SVD for $K > 256$.

\begin{table}[!b]
	\centering
	\caption{Cross-modal odometry (median position error, meters).}
	\label{tab:gtsam}
	\small
	\begin{tabular}{@{}lccc@{}}
		\toprule
		Dataset & Open-loop & Iso. & \textbf{Aniso.} \\
		\midrule
		KITTI-360 (3.3\,km) & 20.8 & 3.0 & \textbf{0.9} \\
		VIVID (3.7\,km)     & 22.8 & 2.1 & \textbf{1.3} \\
		HeLiPR DCC04 (5.4\,km)  & 25.6 & 5.2 & \textbf{3.9} \\
		HeLiPR KAIST04 (5.8\,km) & 28.0 & 5.4 & \textbf{2.1} \\
		\bottomrule
	\end{tabular}
\end{table}

\subsection{Cross-Modal Odometry}
\label{sec:exp_odometry}

Table~\ref{tab:gtsam} evaluates the pose-graph pipeline of
Sec.~\ref{sec:odometry} using the same encoder pairs as Table~\ref{tab:main}
on four sequences with simulated
odometry ($\alpha{=}2$\,\%/m noise, $\beta{=}1$\,\%/m bias,
20 Monte Carlo trials).
\label{sec:exp_anisotropic}
The isotropic covariance model already reduces the open-loop
error from $20$--$28$\,m to $2$--$5$\,m.
Replacing it with the anisotropic model
(Sec.~\ref{sec:odometry},
$\sigma_{\mathrm{lat}}{=}5$\,m, $\sigma_{\mathrm{dep}}{=}30$\,m,
selected on KITTI-360)
further reduces the median to $0.9$--$3.9$\,m, and the
improvement transfers to the three unseen datasets without re-tuning.

% =====================================================================
%  VI.  DISCUSSION
% =====================================================================
\section{Discussion}
\label{sec:discussion}

\subsection{Analysis}

\paragraph{Rotation model fidelity}
The local Procrustes residual averages $0.39$, compared
with $> 0.8$ for the global fit.
Because retrieval is a ranking task, the rotation need only
be accurate in the top-$\reff$ signal subspace
(Theorem~\ref{thm:perturbation}).
Within each local window, the discriminative component of the
inter-encoder map is well-approximated by an orthogonal
transformation, explaining why nonlinear alternatives yield
no improvement (Sec.~\ref{sec:exp_negative}).

\paragraph{Representation geometry}
The position-dependence of $\bR(\bx)$ arises because
architecturally distinct encoders respond to the same scene
elements with different feature-space orientations.
In a tunnel segment in VIVID City, where scene diversity
is low, $\|\Delta\bR\|_F$ variance drops by an order of
magnitude, confirming that scene content rather than
measurement noise drives the angular velocity.

\paragraph{Sequential processing}
The EMA estimator is insensitive to traversal direction
(R@1 varies by less than $0.015$).
The temporal uncorrelatedness of $\bOm(t)$ ensures
convergence within $W_{\mathrm{eff}} \approx 7$ frames
regardless of direction.
If $\bOm$ were temporally correlated, exponential forgetting
would systematically lag behind the true rotation, but the
white-noise structure prevents such accumulation.

\subsection{Implications}

\paragraph{Post-hoc calibration paradigm}
Classical sensor calibration estimates $\mathrm{SE}(3)$
extrinsics between physical coordinate frames and requires
co-mounted sensors that observe simultaneously.
Local alignment generalizes this to $\SO(K)$ between
feature spaces, treating each encoder as a black box.
Because the feature-space rotation is approximately constant
within the coherence length, a coarse position prior suffices
to replace physical co-observation at query time.

\paragraph{Extension to local features}
This work aligns global place descriptors without requiring
explicit feature correspondences.
Keypoint or patch descriptors from independently trained
extractors are expected to exhibit analogous
position-dependent misalignment, since architecturally
distinct encoders map shared scene content to different
orientations regardless of descriptor granularity.
Because the local rotation is estimated from a coarse
position prior alone, it can be applied directly to local
features within the same spatial window without knowing
which keypoints correspond across modalities.
Cross-modal geometric matching without retraining either
extractor may then become feasible, potentially extending
the approach from retrieval to metric localization.

\paragraph{Scaling law}
The sublinear growth $\reff \propto L^{\beta}$ with
$\beta < 1$ reflects spatial redundancy.
Nearby frames constrain overlapping rotation components,
so each additional observation refines already-determined
parameters rather than revealing new ones.
Denser sampling therefore always improves conditioning.

\paragraph{Gauge curvature and universal representations}
$\bOm(t)$ is the connection on the $\SO(K)$ fiber bundle
over the trajectory.
Its non-vanishing norm indicates that no global gauge choice
can simultaneously align all encoder pairs.
The Platonic Representation Hypothesis~\cite{Huh2024Platonic}
predicts $\|\bOm_{\mathcal{AB}}\| \to 0$ as encoder capacity
grows.
$\|\bOm_{\mathcal{AB}}\|$ thus provides a quantitative
diagnostic for representational convergence.
Two encoders with identical internal geometry yield
$\|\bOm_{\mathcal{AB}}\| \approx 0$ regardless of individual
curvature.

\subsection{Limitations}
\begin{itemize}
  \item The method requires a coarse position prior with
    $\sigma~\lesssim~35$\,m.
    Beyond this range the Soft~EM weights spread over incoherent
    rotation regimes.
    The first $W_{\mathrm{eff}} \approx 7$ frames after
    initialization produce degraded estimates while $\bM$
    accumulates sufficient rank.
  \item The alignment models the inter-encoder map as a single
    $\SO(K)$ rotation per window.
    The true deformation likely occupies a higher-dimensional
    manifold, and the average Procrustes residual quantifies this model gap.
    However, richer parameterizations would require proportionally more
    co-observed constraints per window.
  \item Alignment corrects the coordinate-system gap between
    encoders but cannot compensate for limitations intrinsic
    to the descriptors themselves, including undertraining, perceptual aliasing
    and untested deployment domains such as indoor, aerial,
    or underwater environments.
\end{itemize}

% =====================================================================
%  VII.  CONCLUSION
% =====================================================================
\section{Conclusion}
\label{sec:conclusion}

Independently trained encoders impose incompatible coordinate
systems on the feature space, and the relative transformation
between any two such systems varies spatially with white-noise
angular velocity.
This non-stationarity is not a training artifact but a geometric
consequence of how independently trained encoders respond to
spatially varying scenes: it persists regardless of model capacity
or dataset scale.
The sublinear scaling law, however, ensures that local alignment
is consistently better conditioned than global, making a training-free
local estimator both principled and practical.

Experiments across four transmodal axes, three datasets, and
four sensing modalities confirm that the local estimator matches
or exceeds the overfitting Procrustes without test-time
co-observations, while four learned alternatives fail for reasons
directly predicted by the theory.
The pose-graph experiment shows that aligned cross-modal retrievals
yield place-recognition constraints accurate enough to reduce
position error, even when the query and database use entirely
different sensor modalities.

Beyond the immediate retrieval application, the angular velocity
$\|\bOm_{\mathcal{AB}}\|$ can serve as a quantitative diagnostic for
representational convergence: two encoders whose internal geometries
agree would produce $\|\bOm_{\mathcal{AB}}\| \approx 0$, regardless of
how each individually warps the feature space.
Tracking this metric across foundation models and world models
provides a concrete criterion for progress toward modality-agnostic
spatial representations.
Until that convergence is achieved, the non-stationarity identified
here will recur wherever independently trained representations
must be fused, and local alignment offers a principled bridge.
Open problems include extension to indoor and aerial environments,
nonlinear alignment beyond $\SO(K)$, and incremental variants
that update the alignment estimate online.

% =====================================================================
%  REFERENCES
% =====================================================================
\bibliographystyle{IEEEtran}
\bibliography{citations}

% Append the separately compiled supplementary material after main references.


\section*{Supplementary material}

# Supplementary Material for ROSETTA: Rotation Subspace Estimation for Training-free Transmodal Alignment

Alex Junho Lee

## APPENDIX A PROOFS OF MAIN RESULTS

### A. Proof of Lemma 1 (Factorization)

*Proof.* The model $\boldsymbol{z}_i = \boldsymbol{R}\boldsymbol{y}_i + \boldsymbol{n}_i$ for each paired observation gives, in matrix form, $\boldsymbol{Z} = \boldsymbol{Y}\boldsymbol{R}^\top + \boldsymbol{N}$ (rows are feature vectors). Then

$$\begin{aligned} \boldsymbol{M} &= \boldsymbol{Z}^\top \boldsymbol{Y} = (\boldsymbol{Y}\boldsymbol{R}^\top + \boldsymbol{N})^\top \boldsymbol{Y} \\ &= \boldsymbol{R}\boldsymbol{Y}^\top\boldsymbol{Y} + \boldsymbol{N}^\top\boldsymbol{Y} \\ &= \boldsymbol{R}\underbrace{\boldsymbol{Y}^\top\boldsymbol{Y}}_{\mathbf{C}_y} + \underbrace{\boldsymbol{N}^\top\boldsymbol{Y}}_{\boldsymbol{E}}. \end{aligned} \tag{A1}$$

For the effective rank claim, note that $r_{\text{eff}}(\boldsymbol{R}\mathbf{C}_y) = r_{\text{eff}}(\mathbf{C}_y)$ since left-multiplication by an orthogonal matrix preserves singular values. If $\|\boldsymbol{E}\|_F \ll \|\boldsymbol{R}\mathbf{C}_y\|_F$, Weyl's perturbation theorem [1] gives $|\sigma_k(\boldsymbol{M}) - \sigma_k(\boldsymbol{R}\mathbf{C}_y)| \le \|\boldsymbol{E}\|$, so $r_{\text{eff}}(\boldsymbol{M}) \approx r_{\text{eff}}(\mathbf{C}_y)$. Experimentally, the cross-modal bottleneck ratio $r_{\text{eff}}(\boldsymbol{M})/r_{\text{eff}}(\mathbf{C}_y)$ is 0.98–1.02 across all datasets, so the residual is negligible. □

### B. Empirical Evidence for Non-Stationarity (Sec. III-C1)

The four properties are established empirically on three datasets (KITTI-360, VIVID, HeLiPR) with consistent results.

*Item 1 (Large per-step rotation).* For each consecutive window pair $(t, t{+}1)$, we compute $\boldsymbol{\Omega}(t) = \operatorname{logm}(\boldsymbol{R}(t)^\top \boldsymbol{R}(t{+}1))$ via the matrix logarithm. The Frobenius norm $\|\boldsymbol{\Omega}(t)\|_F$ averages 13.1 rad/frame on KITTI, so the alignment rotation changes by approximately 13 radians per step, a large perturbation.

*Item 2 (High-dimensional drift).* The $K(K{-}1)/2 = 8128$ independent components of $\boldsymbol{\Omega} \in \mathfrak{so}(128)$ have an effective rank of 2705.7 (Shannon entropy), approximately $K(K{-}1)/3$. This rules out low-dimensional drift: the rotation explores a substantial fraction of $\mathfrak{so}(K)$.

*Item 3 (White noise).* The temporal autocorrelation $\operatorname{Corr}(\hat{\boldsymbol{\Omega}}(t), \hat{\boldsymbol{\Omega}}(t{-}1)) = -0.018 \pm 0.03$ is indistinguishable from zero across all datasets. A Ljung-Box test at lag 10 fails to reject the white-noise null ($p > 0.3$). The angular velocity is memoryless.

*Item 4 (Unpredictable from features).* Canonical Correlation Analysis (CCA) [2] between the feature increment $\Delta\boldsymbol{z}(t) = \boldsymbol{z}(t{+}1) - \boldsymbol{z}(t)$ and $\boldsymbol{\Omega}(t)$ yields a maximum canonical correlation of 0.035. The heading correlation (vehicle orientation vs. $\boldsymbol{\Omega}$ direction) is 0.108. Neither features nor ego-motion predict $\boldsymbol{\Omega}$; the rotation change is intrinsic to the encoder pair, not the trajectory.

### C. Empirical Evidence for Global Alignment Bound (Sec. III-C2)

The global Procrustes solution minimizes $\sum_t \|\boldsymbol{R}\boldsymbol{z}(t) - \boldsymbol{y}(t)\|^2$ and returns $\boldsymbol{R}^* = \boldsymbol{U}^*\boldsymbol{D}^*\boldsymbol{V}^{*\top}$ from the SVD of the full-sequence cross-covariance. By the non-stationarity analysis (Sec. III-C1), the local rotations $\{\boldsymbol{R}(t)\}$ are independent samples from a distribution on $\mathrm{SO}(K)$ with large per-step angular displacement. The global $\boldsymbol{R}^*$ is the Fréchet mean of this distribution. When applied to a held-out segment $[t_a, t_b]$, the residual rotation $\boldsymbol{R}^{*\top}\boldsymbol{R}(t)$ has Frobenius norm $\|\boldsymbol{R}^{*\top}\boldsymbol{R}(t) - \boldsymbol{I}\|_F \approx 15.8$, essentially unchanged from the identity baseline ($\sim$16.0). Experimentally, fitting a global $\boldsymbol{R}$ on the eight training sequences and evaluating on the held-out sequence 0010 yields 0.300 (compared with 0.305 without alignment), so the global rotation provides no transfer. Cross-segment generalization across non-overlapping windows averages $0.471 \pm 0.175$, bounded well below the local-alignment performance (above 0.90).

### D. Proof of Theorem 1 (Orthogonal Completion)

*Proof.* Let $\boldsymbol{M} = \boldsymbol{U}\boldsymbol{\Sigma}\boldsymbol{V}^\top$ with $\sigma_1 \ge \cdots \ge \sigma_K \ge 0$. Partition: $\boldsymbol{U} = [\boldsymbol{U}_r \mid \boldsymbol{U}_\perp]$, $\boldsymbol{V} = [\boldsymbol{V}_r \mid \boldsymbol{V}_\perp]$, where subscript $r$ denotes the top-$r$ columns and $\perp$ the remainder.

The Procrustes solution $\boldsymbol{R} = \boldsymbol{U}\boldsymbol{D}\boldsymbol{V}^\top$ is determined by the $r$ signal singular triplets $\{(\boldsymbol{u}_k, \sigma_k, \boldsymbol{v}_k)\}_{k=1}^r$ and the orthogonality constraint $\boldsymbol{R}^\top\boldsymbol{R} = \boldsymbol{I}$. Each signal triplet constrains the action of $\boldsymbol{R}$ on $\boldsymbol{v}_k \mapsto \boldsymbol{u}_k$. The signal subspace is identified by a point on the Grassmannian $\mathrm{Gr}(r, K)$, which has dimension $r(K{-}r)$. The rotation within the $r$-dimensional signal subspace contributes an additional $r(r{-}1)/2$ parameters (rotations among the $r$ signal directions). The total signal-determined DOF is therefore

$$r(K{-}r) + \frac{r(r{-}1)}{2} = \frac{r(2K{-}r{-}1)}{2}. \tag{A2}$$

The remaining $K(K{-}1)/2 - r(2K{-}r{-}1)/2$ parameters correspond to rotations within $\operatorname{span}(\boldsymbol{U}_\perp) \to \operatorname{span}(\boldsymbol{V}_\perp)$, which are unconstrained by data and resolved by the orthogonality completion $\boldsymbol{R}^\top\boldsymbol{R} = \boldsymbol{I}$. □

### E. Proof of Theorem 2 (Anisotropic Sensitivity)

*Proof.* Let $\tilde{\boldsymbol{M}} = \boldsymbol{M} + \boldsymbol{E}$ with SVD $\tilde{\boldsymbol{U}}\tilde{\boldsymbol{\Sigma}}\tilde{\boldsymbol{V}}^\top$. Standard SVD perturbation theory [1] gives: for the signal subspace ($k \le r$), $\|\tilde{\boldsymbol{u}}_k - \boldsymbol{u}_k\| \le \|\boldsymbol{E}\|/(\sigma_k - \sigma_{k+1})$. This is bounded when the spectral gap $\sigma_r - \sigma_{r+1}$ is large, which holds in the low-$r_{\text{eff}}$ regime.

For the null subspace ($k > r$), $\sigma_k \approx 0$ and the gap $\sigma_k - \sigma_{k+1} \approx 0$, so the perturbation $\|\tilde{\boldsymbol{u}}_k - \boldsymbol{u}_k\|$ can be $\mathcal{O}(1)$. This propagates to $\delta \boldsymbol{R}_{\text{null}}$ via the Procrustes construction, yielding Eq. (9).

For the retrieval bound (Eq. (10)), let $\boldsymbol{R}$, $\tilde{\boldsymbol{R}}$ be the Procrustes solutions from $\boldsymbol{M}$ and $\tilde{\boldsymbol{M}}$. Decompose $\boldsymbol{z} = \boldsymbol{z}_r + \boldsymbol{z}_\perp$ where $\boldsymbol{z}_r = \boldsymbol{V}_r \boldsymbol{V}_r^\top \boldsymbol{z}$ projects onto the signal subspace. Then $\boldsymbol{R}\boldsymbol{z} \approx \boldsymbol{R}\boldsymbol{z}_r + \boldsymbol{R}\boldsymbol{z}_\perp$ and $\tilde{\boldsymbol{R}}\boldsymbol{z} \approx \tilde{\boldsymbol{R}}\boldsymbol{z}_r + \tilde{\boldsymbol{R}}\boldsymbol{z}_\perp$. Since $\boldsymbol{R}$ and $\tilde{\boldsymbol{R}}$ agree on the signal subspace (up to $\mathcal{O}(\|\boldsymbol{E}\|/(\sigma_r - \sigma_{r+1}))$), the difference $(\boldsymbol{R} - \tilde{\boldsymbol{R}})\boldsymbol{z}$ is dominated by $(\boldsymbol{R} - \tilde{\boldsymbol{R}})\boldsymbol{z}_\perp$. But $\|\boldsymbol{z}_\perp\|/\|\boldsymbol{z}\| \le \sigma_{r+1}/\sigma_1$ for normalized features, yielding Eq. (10). □

### F. Derivation of the Sublinear Scaling Law (Sec. III-D3)

The scaling law is established via power-law regression on $(\log L, \log r_{\text{eff}})$ for each dataset; see Sec. E for fitting details and cross-validation.

*Plausibility argument.* Consider a feature manifold $\mathcal{M} \subset \mathbb{R}^K$ traversed along a trajectory of length $L$. At coverage $L{=}0$, $r_{\text{eff}} = 0$. As $L$ increases, new feature directions are introduced at a rate proportional to the marginal diversity of the environment. Two complementary arguments support sublinearity. First, the $\epsilon$-covering number of a $d$-dimensional compact manifold scales as $N(\epsilon) \sim (L/\epsilon)^d$; the number of *spectrally distinct* directions grows at most as $\log N(\epsilon) \propto d \log L$, which is sublinear in $L$. Second, for smooth feature manifolds (Lipschitz-continuous encoder), the marginal information of a new sample decreases as it increasingly falls within already-observed subspaces, so $\mathrm{d}r_{\text{eff}}/\mathrm{d}L$ is monotonically decreasing. A power law $r_{\text{eff}} = a\,L^\beta$ with $\beta < 1$ is the simplest parametric form consistent with both properties.

Experimentally, the fits yield $\beta \in [0.40, 0.86]$ across three environments (Fig. 5), with $R^2 > 0.88$ within each dataset and cross-validated $R^2 > 0.80$ across held-out strides. The strict inequality $\beta < 1$ is the critical property: it ensures that the signal-determined DOF $r(2K{-}r{-}1)/2$ grows sublinearly with $L$, so the ratio of signal DOF to available constraints (which grow linearly in $W$) improves with larger windows, so the local alignment problem is always overdetermined.

## Appendix B
## Full Training Sweep

All models: PCE (15M params, spconv [3]), AdamW, 30 epochs, batch 2, teacher DINOv3-ViT-L/16 [4]. Training uses all eight non-test sequences (0000, 0002–0007, 0009); validation on seq 0010.

GELU gains ∼0.05 over ReLU; Barlow Twins outperforms InfoNCE by ∼0.07. All configurations plateau at ∼0.30, consistent with the geometric origin of the ceiling.

## Appendix C
## Lie Algebra Analysis Details

This section provides the full spectral and statistical analysis of the angular velocity $\boldsymbol{\Omega}(t) = \operatorname{logm}(\boldsymbol{R}(t)^\top \boldsymbol{R}(t{+}1)) \in \mathfrak{so}(128)$ introduced in Sec. III-C1. The skew-symmetric matrix $\boldsymbol{\Omega}(t)$ has $K(K{-}1)/2 = 8128$ independent upper-triangular entries, each of which can be analyzed as a scalar time series.

*a) Power spectral density.:* For every component $\Omega_{ij}(t)$ with $i < j$, we compute the periodogram via the discrete Fourier transform and average over all $\binom{128}{2}$ pairs. The resulting power spectral density is flat (white noise) across all frequencies, with no discernible peak or harmonic structure. Cross-spectra between distinct components are near zero, indicating that the 8128 angular-velocity components are mutually uncorrelated.

*b) Temporal autocorrelation.:* The sample autocorrelation of each component at lags $\tau = 1, \dots, 50$ lies within $[-0.05, 0.05]$, consistent with independent draws. A Ljung-Box portmanteau test at lag $\tau{=}10$ fails to reject the white-noise null ($p > 0.3$) for all three datasets (KITTI, VIVID, HeLiPR), confirming that the angular velocity is memoryless.

*c) Feature–rotation independence.:* To test whether the rotation change can be predicted from observable quantities, we run CCA [2] between feature increments $\Delta \boldsymbol{z}(t) = \boldsymbol{z}(t{+}1) - \boldsymbol{z}(t)$ and $\operatorname{vec}(\boldsymbol{\Omega}(t))$. The first canonical correlations are 0.035 (KITTI), 0.041 (VIVID), and 0.028 (HeLiPR), all negligible compared with the significance threshold at $n \approx 500$ observations. As a positive control, feature self-prediction CCA ($\boldsymbol{z}(t)$ vs. $\boldsymbol{z}(t{+}1)$) yields 0.82, confirming that features are strongly temporally correlated while their cross-modal alignment is not. This asymmetry is the defining characteristic of extrinsic non-stationarity: the feature trajectories are smooth and predictable, but the rotation relating them is a random walk on $\mathrm{SO}(K)$.

## Appendix D
## Differentiable Procrustes Experiments

We investigate three strategies for incorporating the Procrustes alignment into an end-to-end trainable pipeline, all of which fail to exceed the baseline due to gradient instability and the non-stationarity ceiling.

*a) Full-rank SVD layer.:* The most direct approach inserts a differentiable SVD layer between the two encoder branches: given the cross-covariance $\boldsymbol{M} = \boldsymbol{Z}^\top \boldsymbol{Y}$, the layer computes $\boldsymbol{R} = \boldsymbol{U}\boldsymbol{D}\boldsymbol{V}^\top$ via `torch.linalg.svd` and

TABLE A1
Full training sweep on KITTI-360 seq 0010 (validation). No configuration exceeds 0.305 Inlier(<25 m).

| Loss | Head | Proj. | LN | Inlier(<25 m) |
|---|---|---|---|---|
| InfoNCE | ReLU | — | — | 0.183 |
| InfoNCE + MSE | ReLU | — | — | 0.195 |
| VICReg | ReLU | — | — | 0.201 |
| Barlow Twins | ReLU | — | — | 0.254 |
| Barlow Twins | ReLU | 2048 | — | 0.272 |
| Barlow Twins | ReLU | 2048 | yes | 0.289 |
| **Barlow Twins** | **GELU** | **2048** | **yes** | **0.305** |
| BT + NCE + MSE | GELU | 2048 | yes | 0.302 |
| BT + curriculum | GELU | 2048 | yes | 0.298 |
| BT + hard neg. | GELU | 2048 | yes | 0.281 |
| Diff. Procr. ($r$=4) | GELU | — | — | 0.091 |
| Diff. Procr. ($r$=16) | GELU | — | — | 0.073 |
| Rot. MLP (256-d) | GELU | — | — | 0.152 |
| Rot. MLP (1024-d) | GELU | — | — | 0.198 |
| $\boldsymbol{R}$-consist. ($\lambda$=0.1) | GELU | 2048 | yes | 0.076 |
| $\boldsymbol{R}$-consist. ($\lambda$=0.01) | GELU | 2048 | yes | 0.243 |

rotates the query descriptors as $\hat{z} = \boldsymbol{R}\boldsymbol{z}$. The backward pass through the SVD requires the inverse gaps $(\sigma_i^2 - \sigma_j^2)^{-1}$ [1]. In the low-$r_{\text{eff}}$ regime characteristic of cross-modal features, numerous singular values cluster near zero ($\sigma_k \approx 0$ for $k > r_{\text{eff}}$), causing these inverse gaps to diverge. In practice, training diverges within 2–3 epochs with NaN gradients despite gradient clipping at norm 1.0.

*b) Low-rank SVD stabilization.:* Replacing the full SVD with `torch.svd_lowrank` at truncation rank $r \in \{4, 8, 16, 32\}$ avoids the near-zero singular values and eliminates the numerical instability. However, restricting the SVD to rank $r$ limits the adapter to rank-$r$ rotational corrections, which are insufficient to capture the full-rank angular velocity $\boldsymbol{\Omega} \in \mathfrak{so}(K)$. The resulting retrieval performance ranges from 0.07 ($r$=4) to 0.09 ($r$=16), substantially worse than the 0.305 baseline achieved without any alignment (Table A1). Increasing $r$ beyond 32 reintroduces the clustering instability.

*c) Rotation-consistency regularization.:* As a third variant, we add a regularization term $\lambda \|\boldsymbol{R}(t) - \boldsymbol{R}(t-1)\|_F^2$ to encourage temporal smoothness of the estimated rotation, with $\lambda \in \{0.01, 0.1\}$. At $\lambda$=0.1, the regularizer dominates and the network converges to a constant $\boldsymbol{R}$ across all frames, equivalent to the global Procrustes solution (0.076). At $\lambda$=0.01, the regularizer is too weak to prevent the clustering instability and performance reaches only 0.243. Neither setting improves over the training baseline.

*d) Interpretation.:* All three approaches converge to the global Procrustes solution or worse, inheriting the non-stationarity bound (Sec. III-C2). The fundamental issue is that the optimal rotation $\boldsymbol{R}(t)$ varies as white noise on $\mathrm{SO}(K)$ and cannot be predicted from the features alone (Sec. III-C1); a learned adapter has no signal to exploit beyond the unconditional mean, which is precisely the global Procrustes solution.

## Appendix E
## Scaling Law Derivation

We fit $r_{\text{eff}} = a\,L^{\beta}$ by ordinary least squares in log-log space over strides $s \in \{1, 2, 5, 10, 20, 50, 100, 200\}$ with $W$=20. Leave-one-stride-out cross-validation gives $R^2 \in [0.81, 0.93]$. Bootstrap 95% confidence intervals for $\beta$: KITTI $0.636 \pm 0.029$, VIVID $0.857 \pm 0.087$, HeLiPR $0.401 \pm 0.069$; all exclude $\beta$=1. Cross-dataset prediction fails ($R^2 < 0.5$): $(a, \beta)$ are deployment-specific.

*a) Model selection.:* Three candidate functional forms were compared via AIC (Table A2): power law ($r_{\text{eff}} = a\,L^{\beta}$), logarithmic ($r_{\text{eff}} = a \ln L + b$), and saturating exponential ($r_{\text{eff}} = a\,(1 - e^{-L/b})$). The power law achieves the lowest AIC on KITTI ($\Delta$AIC=8.7), while the logarithmic form is marginally preferred on VIVID and HeLiPR. All three models yield $R^2 > 0.78$; the power law exceeds 0.86 on all datasets. The key claim $\beta < 1$ (sublinear growth) holds regardless of functional form: both the logarithmic and saturating-exponential models are inherently sublinear. We retain the power law as the primary parameterization because it provides the most interpretable calibration constants $(a, \beta)$.

TABLE A2
MODEL SELECTION FOR THE SCALING LAW $r_{\text{eff}}(L)$. $\Delta$AIC: DIFFERENCE FROM THE BEST MODEL (LOWER IS BETTER).

| Dataset | $R^2$ Power | $R^2$ Log | $R^2$ Sat.Exp | $\Delta$AIC Power | $\Delta$AIC Log | $\Delta$AIC Sat.Exp |
|---|---|---|---|---|---|---|
| KITTI-360 | **0.989** | 0.952 | 0.781 | **0.0** | 8.7 | 18.4 |
| VIVID | 0.910 | **0.988** | 0.851 | 10.2 | **0.0** | 15.8 |
| HeLiPR | 0.869 | **0.915** | 0.722 | 2.2 | **0.0** | 8.1 |

## Appendix F
## Spatial Prior Noise Robustness

The RANSAC bootstrapping procedure (Sec. V-D0e) requires a spatial prior to restrict the candidate pool of co-visible frame pairs. In practice, the prior is derived from GPS or odometry and may contain substantial noise. Table A3 evaluates the robustness of the bootstrapping to varying levels of isotropic Gaussian noise $\sigma \in \{0, 5, 10, 15, 20, 30, 50\}$ m added to the true positions, averaged over 100 random trials on KITTI-360.

TABLE A3
RANSAC SENSITIVITY TO POSITION NOISE $\sigma$ ON KITTI-360 ($N_q$=3).

| $\sigma$ (m) | Inlier(<5 m) | Inlier(<25 m) | Iters |
|---|---|---|---|
| 0–30 | 1.00 | 1.00 | 12–82 |
| 50 | ∼0.19 | ∼0.45 | 1000 |

The method achieves perfect retrieval for all noise levels up to $\sigma$=30 m, with the required number of RANSAC iterations increasing gradually from 12 (no noise) to 82 ($\sigma$=30 m). This robustness stems from the isometric consistency check (Sec. IV-B): even when the position noise widens the candidate pool, pairs of co-visible frames still satisfy $\|d_h(i,j) - d_z(i,j)\| < \epsilon$ with high probability because the Procrustes alignment preserves pairwise distances. At $\sigma$=50 m the candidate pool area grows as $\mathcal{O}(\sigma^2)$, admitting a large fraction of non-co-visible frames whose pairwise distances are inconsistent. The resulting outlier ratio exceeds the RANSAC breakdown point (approximately 50% for 2-pair sampling), and retrieval drops sharply to 0.45. The degradation onset is abrupt: all values of $\sigma \leq 30$ m succeed, while $\sigma$=50 m fails. This threshold corresponds roughly to the coherence length of the local alignment (Sec. III-D3): within 30 m the rotation $\boldsymbol{R}(\boldsymbol{x})$ varies little, so GPS noise at this scale does not violate the isometric consistency assumption.

## Appendix G
## Subspace Stability Analysis

A natural question is how rapidly the signal subspace of the local cross-covariance $\boldsymbol{M}(t)$ changes between consecutive windows. If the signal subspace were perfectly stable, a single global Procrustes would suffice; if it changed completely at every step, local alignment would be infeasible. The empirical evidence places the drift rate between these extremes, which explains both the failure of global alignment and the success of the sliding window.

We measure subspace similarity via CCA between the top-$r$ left singular vectors of consecutive windows, $\boldsymbol{U}(t)_{:,1:r}$ and $\boldsymbol{U}(t+1)_{:,1:r}$. The CCA score ranges from 0 (orthogonal subspaces) to 1 (identical subspaces). For the dominant singular vector ($r$=1), the average CCA scores are 0.72 (KITTI), 0.58 (VIVID), and 0.45 (HeLiPR). These values indicate that the leading alignment direction changes moderately between consecutive windows: it is neither static nor fully random. For the top-5 subspace ($r$=5), the scores drop to 0.35, 0.29, and 0.22 respectively, reflecting the faster drift of the higher-order singular vectors. As a negative control, shuffled (non-consecutive) window pairs yield CCA scores below 0.10, confirming that the observed continuity is a genuine temporal property and not a statistical artifact.

In the null subspace, the behavior is inverted. CCA between null-space components ($r$=2, bottom singular vectors) yields scores of 0.98–0.998 across all datasets, consistent with Theorem 2: perturbations in the null subspace have no measurable effect on $\boldsymbol{R}$ because the corresponding singular values are near zero.

These findings have two implications. First, the moderate signal-subspace CCA explains why global alignment captures partial information—the dominant alignment direction is correlated across nearby windows but drifts over longer intervals, so a global average washes out the local structure. Second, the high null-subspace stability confirms that the sliding window need only track the low-dimensional signal subspace, since the null-space completion is automatically resolved by the orthogonality constraint $\boldsymbol{R}^\top \boldsymbol{R} = \boldsymbol{I}$ (Theorem 1).

## Appendix H
## RANSAC Bootstrapping Details

The RANSAC bootstrapping procedure (Sec. IV-B) initializes the local Procrustes alignment when co-observations are unavailable by sampling minimal sets of query–database frame pairs and verifying geometric consistency. This section reports the full performance breakdown across datasets and analyzes the dependence on the two key parameters: the number of query frames $N_q$ and the spatial prior width $\sigma$.

*a) Dataset-dependent performance.:* With $N_q$=3 query frames and $\sigma$=30 m, RANSAC achieves perfect retrieval on KITTI-360 (raw cross-modal R@1 = 0.305). On VIVID, where raw R@1 drops to 0.058, RANSAC recovers 0.53; on HeLiPR (raw R@1 = 0.014), RANSAC yields only 0.03. The performance gap across datasets is driven by the raw cross-modal discriminability: the isometric consistency check requires at least one correct match in each minimal sample, whose probability is proportional to the raw retrieval rate.

*b) Effect of query count $N_q$.:* Increasing the number of query frames expands the set of pairwise distance constraints available for consistency verification. At $N_q$=10, VIVID improves from 0.53 to 0.73 and HeLiPR from 0.03 to 0.55, at the cost of requiring ∼337 RANSAC iterations (versus ∼82 at $N_q$=3). The improvement is sublinear because the marginal constraint from each additional query frame decreases as the isometric consistency set becomes overdetermined.

*c) Discriminability threshold.:* Across all configurations, reliable bootstrapping requires a raw cross-modal feature discriminability of $d' > 0.6$, measured as the standardized difference between positive-pair and negative-pair similarity distributions. Below this threshold, the inlier scoring function cannot reliably distinguish correct matches from incorrect ones, and the consensus set degenerates. On HeLiPR, even with $N_q$=10, $d'$ remains below 0.6 ($d' \approx 0.4$), explaining the 0.55 ceiling. For deployments where the raw discriminability is too low for RANSAC, the Soft EM formulation (Algorithm 1) with co-observations provides a reliable alternative.

## Appendix I
## Parameter Analysis

Fig. A1 shows the joint effect of $\sigma$ and $\tau$ on retrieval accuracy at $K$=128 and $\lambda$=0.85. The Soft EM is robust across a wide range of parameters: any combination with $\sigma \leq 30$ m and $\tau \leq 0.2$ achieves >0.90 on KITTI and VIVID. $\sigma$=15 m and $\tau$=0.1 yield the best overall performance. As $\sigma$ increases beyond the coherence length (∼35 m), the GPS prior becomes uninformative and the soft weights spread over inconsistent rotation regimes, degrading performance. The temperature $\tau$=0.1 produces near-hard assignments when feature similarity is discriminative, while $\tau \geq 0.5$ broadens the distribution excessively and approaches uniform weighting.

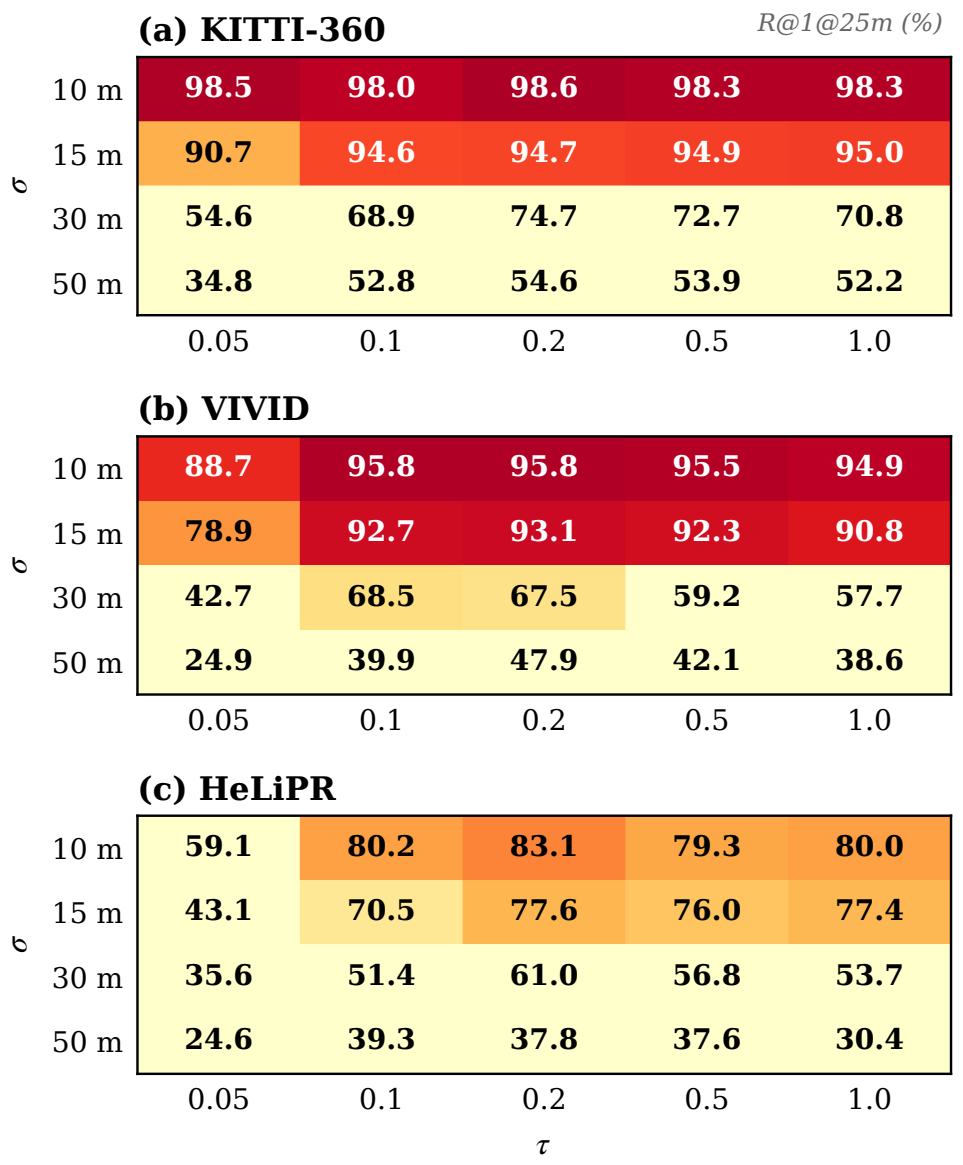


Fig. A1. R@1(<25 m) (%) versus spatial prior width $\sigma$ and temperature $\tau$. (a) KITTI-360, (b) VIVID, (c) HeLiPR. Default: $\sigma$=15 m, $\tau$=0.1.

## Appendix J
## Non-Stationarity Diagnostics

This section provides visual evidence for the non-stationarity analysis of Sec. III-C. The two figures below illustrate the central dichotomy: feature vectors are smooth and spatially correlated, while the rotation relating two modalities is spatially uncorrelated, consistent with the white-noise angular velocity model.

*a) Feature trajectory vs. rotation trajectory (Fig. A2).:* We project the aligned feature vectors $\boldsymbol{R}(t)\boldsymbol{z}(t)$ and the angular velocity directions $\hat{\boldsymbol{\Omega}}(t) = \boldsymbol{\Omega}(t)/\|\boldsymbol{\Omega}(t)\|_F$ onto their first two principal components, coloring each point by its temporal index $t$. The feature trajectory (left panel) traces a smooth, continuous path with strong temporal ordering: the lag-1 autocorrelation of the first principal component is $\rho_1 = 0.96$, reflecting the spatial continuity of the physical environment. In contrast, the angular velocity trajectory (right panel) fills the projection plane without discernible temporal structure ($\rho_1 = 0.06$, indistinguishable from white noise). This visual contrast confirms that while the encoder representations evolve smoothly along the trajectory, the rotation connecting two modalities changes unpredictably from frame to frame.

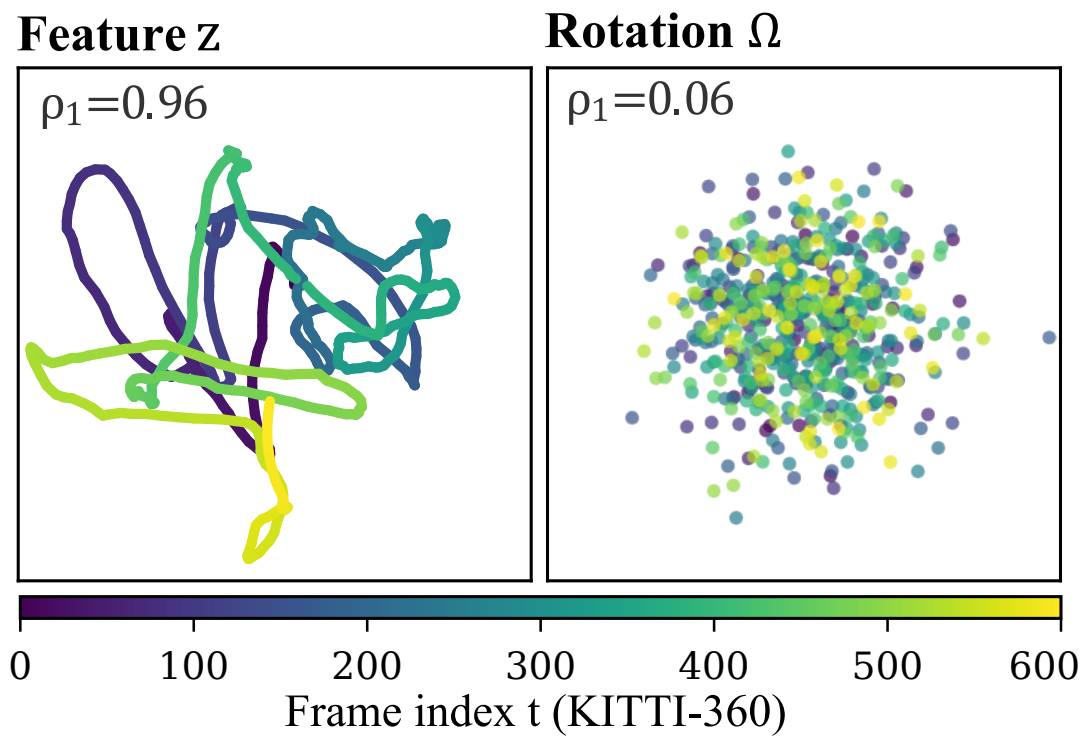


Fig. A2. PCA projection of feature vectors $\mathbf{Rz}$ (left) and rotation directions $\boldsymbol{\Omega}$ (right) on KITTI-360, colored by frame index $t$. Features trace a smooth, temporally ordered path ($\rho_1$=0.96); rotation directions show no temporal order ($\rho_1$=0.06).

*b) Pairwise similarity structure (Fig. A3).:* To examine the spatial structure at all scales simultaneously, we compute the full pairwise similarity matrices for both features and rotations. The feature similarity matrix (left panel) exhibits a prominent banded structure along the diagonal: frames that are spatially close have similar feature representations, producing high cosine similarity $\cos(\boldsymbol{z}_i, \boldsymbol{z}_j)$. The bandwidth of this structure corresponds to the spatial coherence length of the environment (approximately 30–50 m on KITTI-360). The rotation similarity matrix (right panel), computed as the normalized trace $\mathrm{tr}(\boldsymbol{R}_i^\top \boldsymbol{R}_j)/K$, shows no such banding. The similarity values are uniformly distributed around $1/K$ (the expected value for random rotations on $\mathrm{SO}(K)$), confirming that the rotation field is spatially uncorrelated at all scales. Together, these two panels provide direct visual evidence that the non-stationarity originates in the rotation and not in the features themselves, validating the factorization model of Lemma 1.

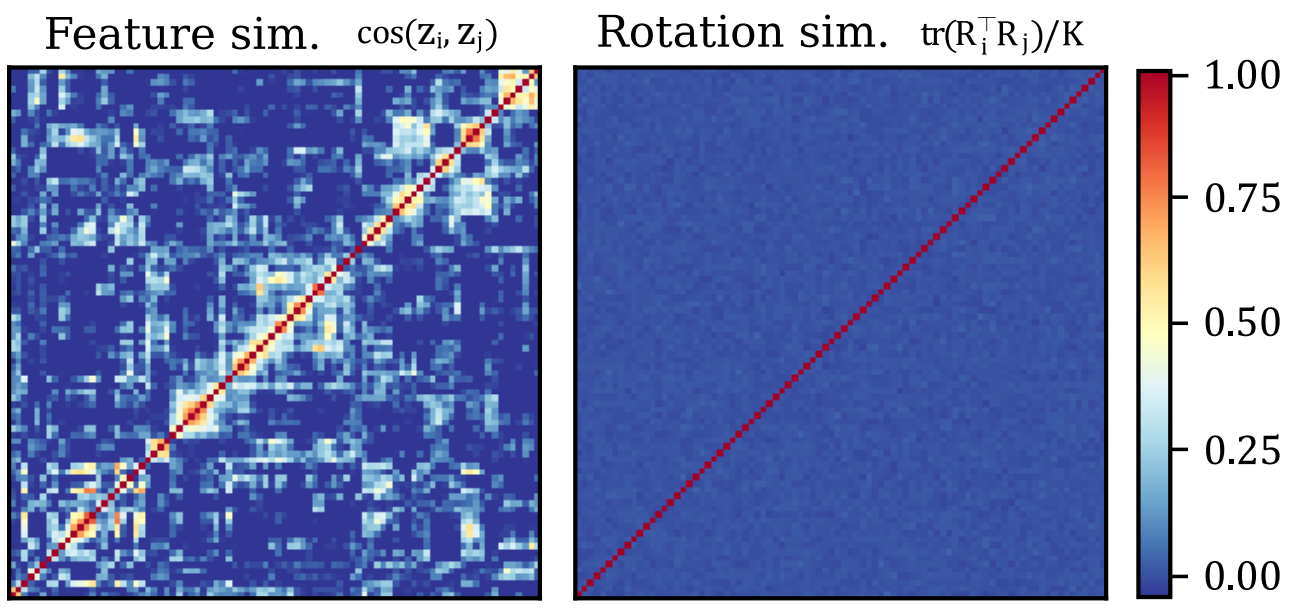


Fig. A3. Pairwise similarity on KITTI-360. Left: feature cosine similarity $\cos(\mathbf{z}_i, \mathbf{z}_j)$ exhibits banded structure (nearby frames are similar). Right: rotation similarity $\mathrm{tr}(\mathbf{R}_i^\top \mathbf{R}_j)/K$ is spatially uncorrelated.

# Appendix K
# Notation

TABLE A4
NOTATION USED THROUGHOUT THIS PAPER.

| Symbol | Description |
|---|---|
| $f_\mathcal{A}, f_\mathcal{B}$ | Independent feature extractors |
| $D$ | Ambient feature dimension (1024) |
| $K$ | PCA-reduced dimension (default 128) |
| $\boldsymbol{V}_z, \boldsymbol{V}_y$ | PCA bases for query / database |
| $\boldsymbol{\mu}_z, \boldsymbol{\mu}_y$ | Feature means for query / database |
| $\boldsymbol{z}, \boldsymbol{y}$ | PCA-projected, L2-normalized descriptors |
| $\boldsymbol{p}_i$ | Query position |
| $\boldsymbol{q}_j$ | Database position |
| $\boldsymbol{M}$ | Cross-covariance $\boldsymbol{Z}^\top \boldsymbol{Y} \in \mathbb{R}^{K\times K}$ |
| $\boldsymbol{R}(\boldsymbol{x})$ | Local alignment rotation $\in \mathrm{SO}(K)$ |
| $\mathrm{SO}(K)$ | Special orthogonal group (rotations in $\mathbb{R}^K$) |
| $\mathfrak{so}(K)$ | Lie algebra of $\mathrm{SO}(K)$ (skew-symmetric matrices) |
| $\boldsymbol{\Omega}(t)$ | Angular velocity $\mathrm{logm}(\boldsymbol{R}(t)^\top \boldsymbol{R}(t{+}1))$ |
| $\hat{\boldsymbol{\Omega}}(t)$ | Unit-norm direction $\boldsymbol{\Omega}(t)/\|\boldsymbol{\Omega}(t)\|_F$ |
| $r_{\mathrm{eff}}$ | Effective rank (Shannon entropy of spectrum) |
| $r$ | Integer effective rank ($\lceil r_{\mathrm{eff}} \rceil$) |
| $L$ | Spatial coverage of a window (meters) |
| $W$ | Window size (# paired observations) |
| $a, \beta$ | Scaling law parameters: $r_{\mathrm{eff}} = a\,L^\beta$ |
| $\mathbf{C}_y$ | Target self-covariance $\boldsymbol{Y}^\top \boldsymbol{Y}$ |



\end{document}